\documentclass{article}

\usepackage[preprint]{neurips_2025}

\usepackage{mathrsfs}
\usepackage{amsmath}

\usepackage[toc]{appendix}
\usepackage[utf8]{inputenc} % allow utf-8 input
\usepackage[T1]{fontenc}    % use 8-bit T1 fonts
\usepackage{hyperref}       % hyperlinks
\usepackage{url}            % simple URL typesetting
\usepackage{booktabs}       % professional-quality tables
\usepackage{amsfonts}       % blackboard math symbols
\usepackage{amssymb}
\usepackage{nicefrac}       % compact symbols for 1/2, etc.
\usepackage{microtype}      % microtypography
\usepackage{xcolor}         % colors
\usepackage{listings}
\usepackage{matlab-prettifier}
\usepackage[capitalize,nameinlink]{cleveref}
\usepackage{microtype}      % microtypography
\usepackage{bm}
\usepackage{multirow}
\usepackage{ulem}
\usepackage{bbm}
\usepackage{graphicx} % 导言区
\usepackage{subfigure}
\usepackage{tabularx}
\usepackage{booktabs}
\usepackage{wrapfig}
\usepackage{amsthm}
\usepackage{algorithmic}
\usepackage[ruled,linesnumbered]{algorithm2e}

\definecolor{pos}{HTML}{5727b0}
\definecolor{neg}{HTML}{b02780}
\definecolor{pl}{HTML}{b0273c}
\definecolor{product_dt}{HTML}{1F77B4}
\definecolor{euclidean_dt}{HTML}{FF7F0E}
\definecolor{orthogonal_dt}{HTML}{2CA02C}
\definecolor{knn}{HTML}{D62728}
\definecolor{perceptron}{HTML}{9467BD}
\definecolor{darkgreen}{HTML}{008000}
\definecolor{mlp}{HTML}{9467BD}
\definecolor{kgcn}{HTML}{8C564B}

\usepackage{graphicx}

\theoremstyle{plain}
\newtheorem{theorem}{Theorem}[section]
\newtheorem{proposition}[theorem]{Proposition}
\newtheorem{lemma}[theorem]{Lemma}
\newtheorem{corollary}[theorem]{Corollary}
\theoremstyle{definition}

\newtheorem{assumption}[theorem]{Assumption}
\theoremstyle{remark}

\usepackage{caption}
\usepackage{pifont}   % 提供勾号和叉号符号
\usepackage{xcolor}   % 提供颜色支持
\newcommand{\cmark}{\textcolor{green!60!black}{\ding{51}}}
\newcommand{\xmark}{\textcolor{red!70!black}{\ding{55}}}

\title{Nora: Normalized Orthogonal Row Alignment
\\for Scalable Matrix Optimizer}

\author{%
	$\text{Jinghui Yuan}^{1}$\thanks{These authors contributed equally.}, \hspace{0.2cm} $\text{Jiaxuan Zou}^{2,4}$\footnotemark[1], \hspace{0.2cm} $\text{Shuo Wang}^{3}$\footnotemark[1], \hspace{0.2cm} $\text{Yong Liu}^{4}$\thanks{Corresponding Author.}, \hspace{0.2cm} $\text{Feiping Nie}^{1}$ \\
	1. School of Artificial Intelligence, Optics and Electronics (iOPEN), \\\ Northwestern Polytechnical University.\\ 
    2. School of Mathematics and Statistics, Xi’an Jiaotong University.\\
    3. Institute for Interdisciplinary Information Sciences, Tsinghua University.\\
    4. Gaoling School of Artificial Intelligence, Renmin University of China.\\
	\texttt{yuanjh@mail.nwpu.edu.cn},
	\texttt{jiaxuanzou@stu.xjtu.edu.cn},
    \texttt{runner21st@gmail.com},\\
    \texttt{liuyonggsai@ruc.edu.cn}, \texttt{feipingnie@gmail.com}
}

\begin{document}

\maketitle
\begin{abstract}
Matrix-based optimizers have demonstrated immense potential in training Large Language Models (LLMs), however, designing an ideal optimizer remains a formidable challenge. A superior optimizer must satisfy three core desiderata: efficiency, achieving Muon-like preconditioning to accelerate optimization; stability, strictly adhering to the scale-invariance inherent in neural networks; and speed, minimizing computational overhead. While existing methods address these aspects to varying degrees, they often fail to unify them, either incurring prohibitive computational costs like Muon, or allowing radial jitters that compromise stability like RMNP. To bridge this gap, we propose Nora, an optimizer that rigorously satisfies all three requirements. Nora achieves training stability by explicitly stabilizing weight norms and angular velocities through row-wise momentum projection onto the row orthogonal complement of the weights. Simultaneously, by leveraging the row block diagonal dominance of the Transformer Hessian, Nora effectively approximates structured preconditioning while maintaining an optimal computational complexity of $\mathcal{O}(mn)$. Furthermore, we prove that Nora is a scalable optimizer and establish its corresponding scaling theorems. With a streamlined implementation requiring only two lines of code, our preliminary experiments validate Nora as an efficient and highly promising optimizer for large-scale training.
\end{abstract}
\section{Introduction}
The training of Large Language Models (LLMs) relies heavily on adaptive optimizers such as Adam \cite{adam,deepseek}. However, such methods treat parameters as flat vectors, thereby overlooking the rich structural interactions inherent in the loss landscape \cite{kimi}. Recently, matrix-based optimizers like Muon have emerged to bridge this gap, achieving state-of-the-art data efficiency through orthogonalized updates \cite{muon}. Despite its remarkable performance, Muon relies on the Newton-Schulz iteration \cite{NSI}, which introduces significant computational overhead of $\mathcal{O}(m^2n)$, where the weight matrix $w \in \mathbb{R}^{m \times n}$. Without loss of generality, we assume $m \le n$.

Subsequent research has sought to approximate this preprocessor. Dong et al.\cite{kdj} identified the row block diagonal dominance of the Transformer Hessian, a phenomenon also noted by \cite{rmnp}, who proposed a simplified preprocessor approximating Muon via diagonal matrices. This led to the development of RMNP (Row-wise Momentum Normalization). Although RMNP is computationally superior, it overlooks a important property of neural networks: weight space symmetry \cite{dcx}. Due to the pervasive use of  BatchNorm \cite{batchnorm}, RMSNorm \cite{rmsnorm} and LayerNorm \cite{layernorm}, network representations exhibit scale invariance \cite{sphericalsj}. Radial updates—those aligning with the weight vector—do not alter the functional output \cite{spherical}. However, radial momentum noise can interfere with the preprocessor's output and the direction of weight learning. Simultaneously, it silently perturbs the weight norm, leading to chaotic oscillations in the effective learning rate and ultimately undermining network training \cite{salimans2016weight}.

Beyond these, other classical or state-of-the-art optimizers typically suffer from one or more deficiencies: they either neglect the use of preconditioners \cite{adafactor,adamp,adamo}, fail to account for the scale-invariance of neural networks \cite{adan,adabelief}, or incur excessive computational complexity \cite{sophia,shampoo}. In fact, many existing methods simultaneously lack several of these essential properties.

Our motivation stems from a commitment to the three core principles of optimizer design: efficiency, stability, and speed. To this end, we introduce Nora. The design of Nora relies on a simple row-wise orthogonality property: row-wise normalization preserves row-wise perpendicularity to the weight. This property enables us to unify preconditioning estimation, scale-invariance, and rapid computation. Specifically, we first project the momentum $v_t$ onto its component that is row-wise perpendicular to $w_t$ to obtain $v_t^{r\perp}$, which ensures training stability through orthogonality. Subsequently, we apply a diagonal preconditioning estimation to $v_t^{r\perp}$ to derive the update $d_t$. This step efficiently approximates the Muon-like preconditioning for enhanced efficiency, yet computationally simplifies to a mere row-wise normalization, thereby simultaneously satisfying both the speed and stability requirements.

Furthermore, leveraging the Maximal Update Parametrization ($\mu$P) framework \cite{mup}, we demonstrate that Nora is a scalable optimizer and derive its optimal learning rate scaling laws. We also provide a rigorous convergence analysis, establishing theoretical guarantees for Nora in non-convex optimization landscapes. Remarkably, the core logic of Nora can be implemented in just two lines of matrix-based code, ensuring its seamless compatibility as a plug-and-play module for any existing optimizer. We evaluate Nora by training LLaMA models \cite{llama} of various scales across a range of learning rates. Our experimental results demonstrate that Nora achieves superior performance in both convergence speed and wall-clock execution time. Our contributions are three-fold:
\begin{itemize}
    \item We propose Nora, a novel optimizer that simultaneously achieves efficiency in preconditioning estimation, stability by respecting the scale invariance of neural networks, and high computational speed during training.
    \item We evaluated Nora by training LLaMA models across various learning rates. Extensive experiments demonstrate that our algorithm achieves better results in terms of both training efficiency and computational wall-clock time.
    \item We rigorously prove through theorems that Nora is a scalable optimizer and provide scaling criteria. We also provide rigorous convergence guarantees for Nora in non-convex optimization, demonstrating that our approach is characterized by theoretical completeness.
\end{itemize}
\section{Notations}
Throughout this paper, $w \in \mathbb{R}^{m \times n}$ denotes the learnable matrix parameters, and $f(w)$ denotes the loss function. Let $w_t$, $g_t$, and $v_t$ denote the parameters, gradients, and momentum at iteration $t$, respectively. The operator $\text{diag}(\cdot)$ keeps the diagonal entries of a square matrix and sets all off-diagonal entries to zero. We denote the $(i,j)$-th entry of the momentum matrix by $v_{t,ij}$ and its $i$-th row by $v_{t,i:}$. We decompose momentum into $v^{\parallel}$ and $v^{\perp}$, the components parallel and orthogonal to $w$. Similarly, $v^{r\parallel}$ and $v^{r\perp}$ denote row-wise parallel and orthogonal components, while $v^{c\parallel}$ and $v^{c\perp}$ denote the corresponding column-wise components. The symbol $\otimes$ denotes the Kronecker product. The notation $(\cdot)^T$ denotes matrix transpose, and $H^{-1}[\cdot]$ denotes the action of a preconditioning operator on a vector or matrix.
\section{Related Work}
\subsection{Matrix-Based Optimizers and Muon}
%Adaptive optimizers such as AdamW \cite{adamw} scale gradients with diagonal preconditioners and therefore ignore off-diagonal curvature in the loss landscape $f(w)$. 
Adaptive optimizers such as AdamW \cite{adamw} treat parameters as vectors and scale gradients component-wise, thereby overlooking the intricate parameter interactions within the loss landscape $f(w)$. Matrix-based optimizers instead treat the weight $w \in \mathbb{R}^{m \times n}$, typically with $m\le n$, as a structured matrix \cite{wen2025}. Let $g_t = \nabla f(w_t)$, and let $v_t = \beta v_{t-1} + (1-\beta) g_t$ denote the exponential moving average of the gradient. Muon orthogonalizes this momentum and uses the update direction $d_t^{\mathrm{M}} = (v_t v_t^T)^{-\frac{1}{2}}v_t$, thereby applying a global structural preconditioner. To avoid the exact inverse square root, Muon uses the Newton--Schulz iteration $X_{k+1} = \frac{1}{2} X_k (3I - X_k^T X_k)$ from a scaled initialization $X_0 \propto v_t$. This update improves the conditioning of the search direction, but the dense matrix multiplications in the Newton--Schulz loop cost $\mathcal{O}(m^2n)$ \cite{fom}, limiting scalability in large LLM layers.
\subsection{Row-Momentum Normalized Preconditioning}
Recent studies show that the layer-wise Hessian matrices of Transformers exhibit strong row-wise block diagonal dominance \cite{why}. Muon can be written as applying the full-spectral preconditioner $H_{M} = (v_t v_t^T)^{\frac{1}{2}} \otimes I_n$, where $v_t$ is the momentum matrix. Row-Momentum Normalized Preconditioning (RMNP) uses the row block dominant structural prior by retaining only the diagonal blocks of the Gram matrix: $H_{R} = (\text{diag}(v_t v_t^T))^{\frac{1}{2}} \otimes I_n$. This approximation reduces matrix orthogonalization to row-wise $\ell_2$ normalization of the momentum: $H_{R}^{-1}[v_t]_{i:} = \frac{v_{t,i:}}{\|v_{t,i:}\|_{2}}$. By eliminating the Newton--Schulz loop, RMNP reduces the per-iteration complexity from $\mathcal{O}(m^2n )$ to $\mathcal{O}(mn)$. However, this algebraic simplification neglects the scale symmetry of neural networks and allows unconstrained radial updates that inject noise into the training dynamics.
\subsection{Scale-invariance in Neural Networks}
The widespread use of BatchNorm, RMSNorm, and LayerNorm makes scale invariance a central property of modern neural networks \cite{path}. For scale-invariant parameters $w$, the loss satisfies $f(w) = f(\lambda w)$ for any $\lambda > 0$. This property implies that the gradient is orthogonal to the weights: $\langle \nabla f(w), w \rangle = 0$. Thus, the magnitude of $w$ is decoupled from the loss value, and learning changes the angular orientation of the parameters \cite{hyperball,hypersphere}. In this setting, the radial momentum component $v^{\parallel}$ interferes with angular optimization and perturbs the weight norm. This perturbation destabilizes the effective learning rate, defined as the relative angular step size, and hinders stable training \cite{lamb}. Our goal is therefore to update along directions $d_t$ satisfying $\langle d_t,w_t\rangle=0$, so that the update remains orthogonal to $w_t$.
\section{Nora Optimizer}
\subsection{Design Concept}
Modern neural networks contain many scale-invariant parameters: changing $\|w_t\|_F$ does not change the loss. Therefore, effective motion is angular motion rather than radial motion, a point sometimes overlooked by designers. \underline{\textbf{Based on the principle of stability}}, the update direction should strictly reside within the tangent space of $w_t$. This approach effectively filters out meaningless radial noise while preventing oscillations in the weight norm $\|w_t\|$. Furthermore, tangential updates allow the learning rate schedule to maintain exclusive control over the angular velocity (i.e., the effective learning rate), since the parameter norm grows at a negligible rate and the preconditioning typically yields an update magnitude determined by the matrix dimensions $\mathbb{R}^{m \times n}$.

Existing preconditioners do not satisfy this requirement. Even when applied to tangential momentum $v_t^{\perp}$, the Muon operator produces an update $d^M_t = (v_t^{\perp}(v_t^{\perp})^T)^{-\frac{1}{2}}v_t^{\perp}$ that deviates from the tangent direction, so $\langle d^M_t, w_t \rangle \neq 0$. RMNP has the same issue: its preconditioner can destroy the orthogonality of the update to $w_t$. SSO \cite{sso} preserves this geometric property, but it relies on a computationally intensive bisection method and is too slow in practice.

Inspired by prior work on the row block diagonal dominance of the Transformer Hessian, we observe that preconditioning can be made compatible with tangential updates by considering $v_t^{r\perp}$. Here, $v_t^{r\perp}$ is obtained by orthogonalizing $v_t$ against each row of $w_t$, ensuring that $\langle v^{r\perp}_{t,i:}, w_{t,i:} \rangle = 0$ for every row $i$.  \underline{\textbf{Based on the principle of efficiency}}, we seek a preconditioner for $v_t^{r\perp}$. Given the row block diagonal dominance of the Transformer Hessian, it follows that $v_t^{r\perp}(v_t^{r\perp})^T$ also exhibits row-diagonal dominance. We therefore adopt the diagonal elements of $v_t^{r\perp}(v_t^{r\perp})^T$ as the preconditioner, specifically $H_N = (\text{diag}(v_t^{r\perp} (v_t^{r\perp})^T))^{\frac{1}{2}} \otimes I_n$. Consequently, the preconditioned update simplifies to $H_N^{-1}[v_t^{r\perp}] = (\text{diag}(v_t^{r\perp} (v_t^{r\perp})^T))^{-\frac{1}{2}} v_t^{r\perp}$. Through algebraic simplification, we arrive at:
    \begin{equation}
\begin{aligned}
    H_N^{-1}[v_t^{r\perp}]_{i:} = ((\text{diag}(v_t^{r\perp} (v_t^{r\perp})^T))^{-\frac{1}{2}} v_t^{r\perp})_{i:} = \frac{v_{t, i:}^{r\perp}}{\| v_{t, i:}^{r\perp} \|_2} .
\end{aligned}
\end{equation}
In other words, applying the preconditioner $H_N[\cdot]$ to $v_t^{r\perp}$ is mathematically equivalent to performing a simple row-wise normalization. Furthermore, the following Equation~\eqref{11111} leads to our long-sought objective: by preconditioning $v_t^{r\perp}$ (which is already row-orthogonal to $w_t$), we simultaneously achieve optimal efficiency while preserving the orthogonality between $v_t^{r\perp}$ and $w_t$. This ensures that the update consistently adheres to the principle of stability.
\begin{equation}
\begin{aligned}
    &\langle H_N^{-1}[v_t^{r\perp}],w_{t}\rangle=\sum_{i=1}^{m}\langle H_N^{-1}[v_t^{r\perp}]_{i:},w_{t,i:}\rangle\\&=\sum_{i=1}^{m}\langle \frac{1}{\| v_{t, i:}^{r\perp} \|_2} v_{t, i:}^{r\perp},w_{t,i:}\rangle=\sum_{i=1}^{m}\frac{1}{\| v_{t, i:}^{r\perp} \|_2}\langle  v_{t, i:}^{r\perp},w_{t,i:}\rangle=\langle (H_N^{-1}[v_t^{r\perp}])^{\perp},w_{t}\rangle=0.
\end{aligned}
\label{11111}
\end{equation}
\underline{\textbf{From the perspective of Speed}}, Nora maintains high efficiency and stability while introducing minimal computational overhead. The additional computation consists solely of row-wise projections of $v_t$ onto the row tangent space of $w_t$, which are element-wise operations that consume virtually no additional time compared to RMNP. Furthermore, by exploiting the row block diagonal dominance of the Transformer Hessian, Nora replaces the computationally intensive Newton-Schulz (NS) iterations used in Muon with simple row-wise normalization. This effectively reduces the overall time complexity and ensures optimal training throughput.

Furthermore, updating along the direction of $(\text{diag}(v_t^{r\perp} (v_t^{r\perp})^T))^{-\frac{1}{2}} v_t^{r\perp}$ offers additional advantages. First, the discrete tangential updates ensure that the norm of each row increases only at a second-order rate. This growth is both monotonic—guaranteeing steady progression without erratic fluctuations—and gentle, remaining strictly controllable throughout training. \underline{\textbf{Moreover}}, the row-wise orthogonality of $d_t$ relative to $w_t$ ensures a uniform growth rate across all rows of $w_t$. By upholding the geometric consistency of scale-invariant parameters, this mechanism effectively mitigates internal covariate shift and maintains the representation balance of the neural network.

It is also noteworthy that, due to the row-block diagonal dominance of the Transformer Hessian, Nora consistently performs row-wise normalization on the momentum $v_t^{r\perp}$ regardless of whether $m \le n$ or $m \ge n$. This stands in stark contrast to Muon, which restricts its Newton-Schulz iterations to the square matrix of the smaller dimension. The comprehensive procedure of the Nora algorithm is detailed in Algorithm \ref{suanfa1}. Furthermore, Table \ref{yuanze} presents a comparative summary of how various optimizers adhere to the three core design principles. 
\begin{algorithm}[H]
\caption{The Nora Optimizer}
\label{suanfa1}
\begin{algorithmic}
\REQUIRE Layer Weight $w_t \in \mathbb{R}^{m \times n}$, momentum $v_t \in \mathbb{R}^{m \times n}$, learning rate $\eta_t$ at step $t$, momentum coefficient $\beta$, and weight decay coefficient $\lambda$ (Default $\lambda$ = 0).
\STATE Initialize $v_0 \gets \mathbf{0} \in \mathbb{R}^{m \times n}, t \gets 0$.
\FOR{each step}
    \STATE $g_t \gets \nabla f(w_t)$
    \STATE $v_t \gets  \beta v_{t-1} +(1-\beta) g_t$
    \STATE $v_{t,i:}^{r\perp} \gets v_{t,i:}-\frac{\langle v_{t,i:},w_{t,i:} \rangle}{\|w_{t,i:}\|^2_2}w_{t,i:}$   \hfill \COMMENT{For all Row}
    \STATE $d_t \gets (\text{diag}(v_t^{r\perp} (v_t^{r\perp})^T))^{-\frac{1}{2}} v_t^{r\perp}$       \hfill \COMMENT{Row Normalization}
    \STATE $w_{t+1} \gets w_t - \eta_t  (d_t + \lambda w_t)$ 
\ENDFOR
\end{algorithmic}
\end{algorithm}

\begin{table*}[t]
  \centering
    \caption{Comparison of Optimizers Based on the Three Design Principles}
\resizebox{0.9\linewidth}{!}
{
    \begin{tabular}{c|c|c|c|c|c|c|c}
    \toprule
Optimizer& Nora& Muon \cite{muon}&RMNP \cite{rmnp}&SSO \cite{sso} &SGD \cite{sgd} &Adam \cite{adam} &AdamP \cite{adamp} \\
                        \midrule
Efficiency&\cmark & \cmark & \cmark &\cmark & \xmark& \xmark& \xmark \\ \midrule 
Stability&\cmark & \xmark& \xmark& \cmark &\xmark& \xmark& \cmark  \\
\midrule
Speed&\cmark & \xmark& \cmark & \xmark&\cmark & \cmark & \cmark  \\
\bottomrule
\end{tabular}%
}
\label{yuanze}
\end{table*}%

\subsection{Theoretical Analysis}
We now analyze Nora as a scalable optimizer and establish non-convex convergence guarantees. For a matrix \(x\in\mathbb{R}^{m\times n}\), define:
\[
  \|x\|_{1,2}:=\sum_{i=1}^{m}\|x_{i:}\|_2,
  \qquad
  \|x\|_{\infty,2}:=\max_{1\le i\le m}\|x_{i:}\|_2 .
\]
For any matrix \(w\) with nonzero rows, a matrix \(z\) is row-wise perpendicular to \(w\) if \(\langle z_{i:},w_{i:}\rangle=0\) for every row \(i\). The row-wise perpendicular projection is:
\begin{equation}
    [\mathcal{P}^{r\perp}_{w}(x)]_{i:}
    :=
    x_{i:}
    -
    \frac{\langle x_{i:},w_{i:}\rangle}{\|w_{i:}\|_2^2}\,w_{i:}.
\label{eq:row-perpendicular-projection}
\end{equation}
We use row-wise normalization with the convention \(0/0=0\):
\begin{equation}
[\operatorname{RN}(x)]_{i:}:=
\begin{cases}
\dfrac{x_{i:}}{\|x_{i:}\|_2}, & x_{i:}\neq 0,\\[0.5em]
0, & x_{i:}=0.
\end{cases}
\label{eq:row-normalization}
\end{equation}
Throughout this subsection, the scalar \(m\) denotes the number of rows of \(w\in\mathbb{R}^{m\times n}\), and \(v_t\) denotes the momentum sequence. For \(t=0,\ldots,T-1\), we analyze the core Nora update without decoupled weight decay:
\begin{equation}
\begin{aligned}
    g_t &= \nabla f(w_t;\xi_t),
    &v_{t+1} &= \beta v_t+(1-\beta)g_t,
    &v_0 &= 0,\\
    v_{t+1}^{r\perp} &= \mathcal{P}^{r\perp}_{w_t}(v_{t+1}),
    &d_t &= \operatorname{RN}(v_{t+1}^{r\perp}),
    &w_{t+1} &= w_t-\eta d_t .
\end{aligned}
\label{eq:nora-core-theory}
\end{equation}

\subsubsection{Scaling for Nora}
In this section, we address the most fundamental engineering question in LLM training: How should the learning rate \(\eta\) of Nora scale with the model width? According to the core principles of Maximal Update Parametrization (\(\mu\)P), an optimal optimizer must ensure that the change in hidden layer activations, \(\Delta h\), remains on the scale of \(\Theta(1)\) as the network width increases (\(n \to \infty\)). This ensures the network neither fails to learn features due to vanishing updates nor suffers from numerical explosion. We derive the rigorous learning rate formula required for Nora to satisfy this limit through the following theorem:
\begin{theorem}[Nora Scaling under the Scaling Hypothesis]
    Consider a neural network layer defined by \(h = wx\), where the weights \(w \in \mathbb{R}^{m \times n}\). Assume the input activation \(x \in \mathbb{R}^n\) satisfies the scaling hypothesis under standard deep learning initialization: \(\|x\|_2 \le \gamma \sqrt{n}\), where \(\gamma = \Theta(1)\) is a constant. Suppose the parameters are updated using Nora: \(w_{t+1} = w_t - \eta_t d_t\). Then \(|\Delta h_i| \le \eta_t \gamma \sqrt{n}\). To achieve the stable feature learning limit required by \(\mu\)P theory---specifically, to ensure the activation update magnitude satisfies \(|\Delta h_i| = \Theta(1)\)---the learning rate \(\eta_t\) for Nora must follow the width-scaling rule, \(\eta_t \propto 1/\sqrt{n}\).
\end{theorem}
The theorem provides an upper bound on the forward update, \(|\Delta h_i| \le \eta_t \gamma \sqrt{n}\). This bound ensures training stability but raises a more precise question: do Nora's projection and normalization operations substantially weaken the effective gradient components in practice, causing \(\Delta h_i\) to fall significantly below \(\sqrt{n}\) or even vanish? To show that Nora achieves Maximal Update Parametrization (\(\mu\)P), we must evaluate the true order of \(\Delta h_i\) in the high-dimensional limit \(n \to \infty\). To this end, we use the following theorem.
\begin{theorem}[Asymptotic Convergence]
    Consider a hidden layer \(h = wx\), where \(w \in \mathbb{R}^{m \times n}\). Under the standard infinite-width random initialization hypothesis, assume the row vectors of \(w\) follow \(w_{i:} \sim \mathcal{N}(0, \sigma_w^2 I_n/n)\), and the components of the input activation \(x \in \mathbb{R}^n\) are independent with zero mean and variance \(\sigma_x^2 = \Theta(1)\). Let the error signal backpropagated to this layer be \(\delta_i = \Theta(1)\), and let the vanilla gradient without momentum be \(g_{i:} = \delta_i x^\top\). If Nora generates the update direction \(d_{i:} = \operatorname{RN}(\mathcal{P}_{w}^{r\perp}(g_{i:}))\), then as the network width \(n \to \infty\), the inner product between the update direction and the input converges in probability to \(\langle d_{i:}, x \rangle \xrightarrow{p} \operatorname{sgn}(\delta_i) \sigma_x \sqrt{n}\). Therefore, to obtain non-trivial feature learning with \(\Delta h_i = \Theta(1)\), Nora must use \(\eta = \eta_0/\sqrt{n}\).
\end{theorem}

\subsubsection{Convergence Analysis}
\label{subsec:theory}
We next prove non-convex convergence guarantees for Nora. Define the natural filtration as:
\[
  \mathcal{F}_t:=\sigma(w_0,\xi_0,\ldots,\xi_{t-1}).
\]
It is the sigma-algebra generated before sampling \(\xi_t\). Equivalently, \(w_t\) and \(v_t\) are \(\mathcal{F}_t\)-measurable, and the conditional expectations below are taken only over the current stochastic gradient. We measure stationarity by the row-wise projected gradient:
\begin{equation}
\mathcal{G}_t:=\mathcal{P}^{r\perp}_{w_t}\bigl(\nabla f(w_t)\bigr).
\label{eq:nora-projected-gradient}
\end{equation}
This is the relevant first-order signal for Nora because each row of \(d_t\) is perpendicular to the corresponding row of \(w_t\). Hence its descent inner product depends only on the row-wise perpendicular component of the true gradient, not on the radial component removed by the projection. Under row-wise scale invariance, \(\mathcal{G}_t\) coincides with \(\nabla f(w_t)\), so the projected stationarity measure reduces to the standard first-order one.

Compared with RMNP, Nora inserts one row-wise perpendicular projection before row-wise normalization. This additional step preserves the RMNP proof structure; the main new ingredient is the non-expansiveness of \(\mathcal{P}^{r\perp}_{w_t}\) in the norms used below.

\begin{assumption}[Smoothness]
\label{ass:nora-smooth}
The objective \(f:\mathbb{R}^{m\times n}\to\mathbb{R}\) satisfies one of the following conditions:
\begin{itemize}
    \item[(a)] (\emph{Frobenius smoothness}) There exists \(L_F>0\) such that, for all \(w,w'\),
    \begin{equation}
    \|\nabla f(w)-\nabla f(w')\|_F\le L_F\|w-w'\|_F .
    \end{equation}
    \item[(b)] (\emph{Matched \((\infty,2)\)-smoothness}) There exists \(L_{\infty,2}>0\) such that, for all \(w,w'\),
    \begin{equation}
    \|\nabla f(w)-\nabla f(w')\|_{1,2}\le L_{\infty,2}\|w-w'\|_{\infty,2} .
    \end{equation}
\end{itemize}
\end{assumption}

% \begin{assumption}[Unbiased stochastic gradients]
% \label{ass:nora-unbiased}
% For all \(t\),
% \begin{equation}
%     \mathbb{E}[g_t\mid \mathcal{F}_t]=\nabla f(w_t).
% \end{equation}
% \end{assumption}

% \begin{assumption}[Bounded variance]
% \label{ass:nora-var}
% There exists \(\sigma>0\) such that, for all \(t\),
% \begin{equation}
% \mathbb{E}\!\left[\|g_t-\nabla f(w_t)\|_F^2\mid \mathcal{F}_t\right]\le \frac{\sigma^2}{B},
% \end{equation}
% where \(B\) is the batch size.
% \end{assumption}

% \begin{assumption}[Lower bounded objective]
% \label{ass:nora-lower}
% The objective is bounded below by \(f^\star\). We write \(\Delta:=f(w_0)-f^\star\).
% \end{assumption}

\begin{theorem}[Nora under matched \((\infty,2)\)-smoothness]
\label{thm:nora-main}
Suppose Assumptions~\ref{ass:nora-smooth}(b) hold. If Nora uses a constant step size \(\eta_t=\eta\), then:
\begin{equation}
\frac{1}{T}\sum_{t=0}^{T-1}\mathbb{E}\bigl[\|\mathcal{G}_t\|_{1,2}\bigr]
\le
\frac{\Delta}{T\eta}
+
2\left[
\left(1-\frac{1}{T}\right)\frac{L_{\infty,2}\eta\beta}{1-\beta}
+
\frac{\sqrt{m}\sigma}{\sqrt{B}}
\sqrt{\frac{1-\beta}{1+\beta}}
\right]
+
\frac{L_{\infty,2}\eta}{2}.
\label{eq:nora-main-bound}
\end{equation}
Here \(m\) is the row dimension of \(w\). With the choice:
\begin{equation}
\eta=\sqrt{\frac{(1-\beta)\Delta}{L_{\infty,2}T}},
\qquad
1-\beta=
\min\left\{
\frac{\sqrt{L_{\infty,2}\Delta}}{2\sqrt{m}\sigma\sqrt{T}},
\,1
\right\},
\label{eq:nora-main-params}
\end{equation}
Nora reaches an \(\epsilon\)-stationary point in the projected \(\|\cdot\|_{1,2}\) sense after:
\begin{equation}
    T=\mathcal{O}\!\left(mL_{\infty,2}\sigma^2\Delta\,\epsilon^{-4}\right)
\end{equation}
iterations.
\end{theorem}

\begin{proposition}[Frobenius-smooth counterparts]
\label{prop:nora-frob} Under the same conditions of Theorem \ref{thm:nora-main}, if Nora uses a constant step size \(\eta_t=\eta\), then:
\begin{equation}
\frac{1}{T}\sum_{t=0}^{T-1}\mathbb{E}\bigl[\|\mathcal{G}_t\|_F\bigr]
\le
\frac{\Delta}{T\eta}
+
(\sqrt{m}+1)\left[
\left(1-\frac{1}{T}\right)\frac{L_F\eta\sqrt{m}\beta}{1-\beta}
+
\frac{\sigma}{\sqrt{B}}
\sqrt{\frac{1-\beta}{1+\beta}}
\right]
+
\frac{L_F\eta m}{2}.
\label{eq:nora-frob-bound}
\end{equation}
The corresponding projected \(\|\cdot\|_{1,2}\) bound is:
\begin{equation}
\frac{1}{T}\sum_{t=0}^{T-1}\mathbb{E}\bigl[\|\mathcal{G}_t\|_{1,2}\bigr]
\le
\frac{\Delta}{T\eta}
+
2\left[
\left(1-\frac{1}{T}\right)\frac{L_F\eta m\beta}{1-\beta}
+
\frac{\sqrt{m}\sigma}{\sqrt{B}}
\sqrt{\frac{1-\beta}{1+\beta}}
\right]
+
\frac{L_F\eta m}{2}.
\label{eq:nora-frob-l12-bound}
\end{equation}
Consequently, Nora reaches an \(\epsilon\)-stationary point in either projected measure after:
\begin{equation}
T=\mathcal{O}\!\left(m^2L_F\sigma^2\Delta\,\epsilon^{-4}\right)
\end{equation}
iterations.
\end{proposition}

\begin{corollary}[Standard first-order stationarity under row-wise scale invariance]
\label{cor:nora-row-scale}
Assume, in addition, that \(f\) is row-wise scale invariant:
\begin{equation}
f(Dw)=f(w),\qquad \forall w\in\mathbb{R}^{m\times n},\ \forall D\succ 0\ \text{diagonal}.
\end{equation}
Then \(\nabla f(w_t)\) is row-wise perpendicular to \(w_t\) for all \(t\), i.e.,
\begin{equation}
\mathcal{P}^{r\perp}_{w_t}\bigl(\nabla f(w_t)\bigr)=\nabla f(w_t).
\end{equation}
Therefore, Theorem~\ref{thm:nora-main} and Proposition~\ref{prop:nora-frob} hold with \(\mathcal{G}_t\) replaced by \(\nabla f(w_t)\).
\end{corollary}

\paragraph{Proof sketch.}
The proof follows the RMNP descent argument and adds one projection lemma: for fixed \(w_t\), the row-wise perpendicular projector \(\mathcal{P}^{r\perp}_{w_t}\) is non-expansive in \(\|\cdot\|_F\), \(\|\cdot\|_{1,2}\), and \(\|\cdot\|_{\infty,2}\). Because each row of \(d_t\) is perpendicular to the corresponding row of \(w_t\),
\begin{equation}
\langle \nabla f(w_t),d_t\rangle
=
\left\langle \mathcal{P}^{r\perp}_{w_t}(\nabla f(w_t)),d_t\right\rangle
=
\langle \mathcal{G}_t,d_t\rangle.
\end{equation}
Define the projected momentum tracking error as:
\[
 e_t:=v_{t+1}^{r\perp}-\mathcal{G}_t .
\]
The row-normalization identities give:
\begin{equation}
\langle \mathcal{G}_t,d_t\rangle
\ge
\|\mathcal{G}_t\|_{1,2}-2\|e_t\|_{1,2},
\qquad
\langle \mathcal{G}_t,d_t\rangle
\ge
\|\mathcal{G}_t\|_F-(\sqrt{m}+1)\|e_t\|_F.
\end{equation}
Combining these inequalities with smoothness-based descent and the standard momentum recursion yields the stated bounds. Full proofs and standard assumptions are deferred to Appendix~\ref{app:nora-proof}.

\subsection{Compared with Mano}
In this section, we delineate the distinctions between Mano \cite{mano} and Nora. Mano alternates between row-wise and column-wise perpendicular projections, followed by normalization. Specifically, Mano accumulates momentum $v_t$ in the standard manner. During odd iterations, it removes the row-wise radial component to obtain $v_t^{r\perp}$ and applies row normalization: $d_t = \text{RN}(v_t^{r\perp})$. During even iterations, it removes the column-wise radial component to obtain $v_t^{c\perp}$ and applies column normalization: $d_t = \text{CN}(v_t^{c\perp})$.

The design philosophy of Nora diverges significantly from that of Mano. Nora originates from the properties expected by the optimizer itself, while Mano comes from Riemannian optimization \cite{intro2m,rim,rfk}. In practical implementation, Nora is more streamlined as it eliminates the need to track iteration parity (odd vs. even steps). More importantly, Nora is explicitly designed to leverage the row block diagonal dominance characteristic of the Transformer Hessian. By removing the heuristic column-wise orthogonal projection and normalization used in Mano, Nora demonstrates superior properties in ablation studies. These empirical results suggest that the row block diagonal dominance of the Transformer Hessian is a critical structural prior that must be prioritized.
\section{Experiments}
\label{sec:experiments}

We evaluate Nora on autoregressive language modeling with LLaMA-style Transformer models at two scales: 60M and 135M parameters. We compare Nora with three matrix-based optimizers, Muon, Mano, and RMNP, under the same data pipeline, tokenizer, context length, global batch size, learning-rate schedule, precision, evaluation cadence, and checkpointing cadence. Detailed model and training configurations are deferred to Appendix~\ref{app:exp_details}.

The compared optimizers differ only in the optimizer rule, the matrix learning-rate sweep grid, and weight decay. For all non-matrix parameter groups, we use the same auxiliary Adam setting within each model size. For matrix-shaped parameters, we tune the matrix learning rate over optimizer-specific grids and select the best run by validation loss. Muon uses a slightly different sweep range from the other matrix optimizers because it is more sensitive to large matrix learning rates in our preliminary runs. The full sweep grids are listed in Appendix~\ref{app:exp_details}. Importantly, Nora uses weight decay $0$ in all runs, whereas Muon, Mano, and RMNP use weight decay $0.1$. We treat this as part of Nora's optimizer configuration and report it explicitly in all result tables.

\subsection{Best Results from Different Optimizers}

Table~\ref{tab:main_lm_results} reports the main language-modeling results. For each optimizer and model size, we report the matrix learning rate selected from the sweep, the validation loss at the selected setting, and validation perplexity. This format separates peak validation performance from the final checkpoint quality.

\begin{table*}[h]
\centering
\caption{Main language-modeling results. For each optimizer and model size, the matrix learning rate is selected by validation loss over the sweep grid in Appendix~\ref{app:exp_details}. Lower validation loss and perplexity are better. Best results are shown in bold.}
\label{tab:main_lm_results}
\resizebox{\linewidth}{!}{
\begin{tabular}{lccccccc}
\toprule
\multirow{2}{*}{Optimizer} & \multirow{2}{*}{Weight decay} 
& \multicolumn{3}{c}{60M} 
& \multicolumn{3}{c}{135M} \\
\cmidrule(lr){3-5}
\cmidrule(lr){6-8}
& & Best matrix LR & Val. loss & Val. ppl. 
& Best matrix LR & Val. loss & Val. ppl. \\
\midrule
Muon & 0.1 & 0.01  & 3.44 & 31.09 & 0.01  & 3.142 & 23.17 \\
Mano & 0.1 & 0.004 & 3.39 & 29.55 & 0.003 & 3.097 & 22.13 \\
RMNP & 0.1 & 0.004 & 3.41 & 30.12 & 0.01  & 3.112 & 22.46 \\
Nora & 0.0 & 0.004 & \textcolor{red!80!black}{\textbf{3.37}} & \textcolor{red!80!black}{\textbf{28.94}} & 0.003 & \textcolor{red!80!black}{\textbf{3.079}} & \textcolor{red!80!black}{\textbf{21.74}} \\
\bottomrule
\end{tabular}
}
\end{table*}

\begin{table}[htbp]
\caption{
Comparison of Mano and Nora (Model: 135M, $\text{Weight decay}=0$)
}
\centering
\begin{tabular}{c|cc|cc|cc|cc}
\hline
\multirow{2}{*}{Optimizer} & \multicolumn{2}{c}{0.003}              & \multicolumn{2}{c}{0.005}              & \multicolumn{2}{c}{0.01}               & \multicolumn{2}{c}{0.02}               \\ \cline{2-9} 
                            & \multicolumn{1}{c|}{ppl} & loss  & \multicolumn{1}{c|}{ppl} & loss  & \multicolumn{1}{c|}{ppl} & loss  & \multicolumn{1}{c|}{ppl} & loss  \\ \hline
Mano                        & \multicolumn{1}{c|}{22.13}      & 3.097 & \multicolumn{1}{c|}{22.14}      & 3.098 & \multicolumn{1}{c|}{23.34}      & 3.150 & \multicolumn{1}{c|}{25.07}      & 3.222 \\ \hline
Nora                        & \multicolumn{1}{c|}{21.74}      & 3.079 & \multicolumn{1}{c|}{21.86}      & 3.085 & \multicolumn{1}{c|}{22.43}      & 3.111 & \multicolumn{1}{c|}{23.39}      & 3.152 \\ \hline
\end{tabular}
\label{xrsy}
\end{table}

\begin{figure}[h]
\centering
\begin{minipage}{0.48\linewidth}
    \centering
    \includegraphics[width=\linewidth]{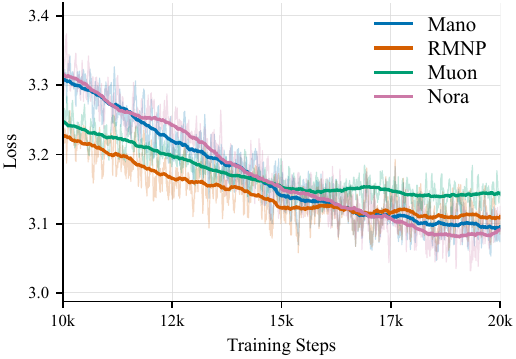}
\end{minipage}
\hfill
\begin{minipage}{0.48\linewidth}
    \centering
    \includegraphics[width=\linewidth]{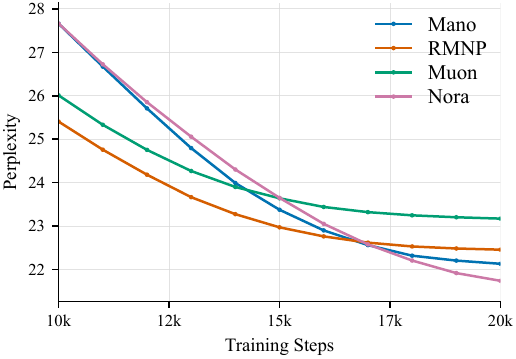}
\end{minipage}
\caption{
Training dynamics on the 135M model. Left: loss over training steps.
Right: perplexity over training steps. Nora continues to improve late in
training and finishes with the lowest loss and perplexity.
}
\label{fig:135m_training_dynamics}
\end{figure}

Figure~\ref{fig:135m_training_dynamics} shows the 135M training dynamics from
10k to 20k steps. RMNP has the strongest early-stage loss and perplexity, while
Muon improves more slowly and plateaus at a higher value. Mano and Nora start
from higher loss and perplexity, but both continue to improve after 15k steps.
Nora shows the clearest late-stage improvement: it overtakes the other methods
near the end of training and reaches the lowest final loss and perplexity. This indicates that Nora uses the same training budget more \underline{\textbf{efficiently}},
achieving better optimization quality without requiring additional training
steps. This
trend is consistent with Table~\ref{tab:main_lm_results}, where Nora achieves
the best 135M validation loss and perplexity among the compared optimizers.

\subsection{Runtime Comparison}
We also compare the cost of row normalization with Newton--Schulz orthogonalization
on representative matrix shapes from LLaMA-style models. The benchmark uses CUDA
with bfloat16 precision, 20 warmup iterations, and 200 measured iterations.
Newton--Schulz uses five iterations. As shown in \Cref{tab:runtime_comparison},
row normalization is consistently cheaper than Newton--Schulz. The gap is already
about one order of magnitude for small and medium matrices, and becomes much larger
for the 1B-scale MLP matrices, where Newton--Schulz is more than 70 times slower. This highlights the substantial \underline{\textbf{speed}} advantage of row normalization in practical
training settings. Results support the practical motivation for replacing matrix-level
orthogonalization with a row-wise normalization step.

\begin{table}[h]
\centering
\caption{
Runtime comparison between row normalization and five-step Newton--Schulz
orthogonalization on representative LLaMA-style matrix shapes. Times are mean
kernel runtimes in milliseconds under CUDA and bfloat16 precision.
}
\label{tab:runtime_comparison}
\resizebox{\linewidth}{!}{
\begin{tabular}{cccccc}
\toprule
Model scale & Matrix shape & Representative layer
& Row normalization (ms) & NS(5) (ms) & NS / row-norm \\
\midrule
60M  & $512 \times 512$    & attention: hidden $\times$ hidden
     & 0.0689 & 0.6554 & 9.52$\times$ \\
60M  & $1376 \times 512$   & MLP: intermediate $\times$ hidden
     & 0.0682 & 0.6853 & 10.05$\times$ \\
60M  & $512 \times 1376$   & MLP: hidden $\times$ intermediate
     & 0.0688 & 0.6623 & 9.63$\times$ \\
135M & $768 \times 768$    & attention: hidden $\times$ hidden
     & 0.0686 & 0.6520 & 9.50$\times$ \\
135M & $2048 \times 768$   & MLP: intermediate $\times$ hidden
     & 0.0692 & 0.6871 & 9.93$\times$ \\
135M & $768 \times 2048$   & MLP: hidden $\times$ intermediate
     & 0.0691 & 0.6534 & 9.46$\times$ \\
350M & $1024 \times 1024$  & attention: hidden $\times$ hidden
     & 0.0674 & 0.6397 & 9.49$\times$ \\
350M & $2816 \times 1024$  & MLP: intermediate $\times$ hidden
     & 0.0687 & 0.8625 & 12.56$\times$ \\
350M & $1024 \times 2816$  & MLP: hidden $\times$ intermediate
     & 0.0675 & 0.6941 & 10.28$\times$ \\
1B   & $2048 \times 2048$  & attention: hidden $\times$ hidden
     & 0.0684 & 2.0552 & 30.06$\times$ \\
1B   & $5461 \times 2048$  & MLP: intermediate $\times$ hidden
     & 0.0985 & 6.9985 & 71.02$\times$ \\
1B   & $2048 \times 5461$  & MLP: hidden $\times$ intermediate
     & 0.1084 & 7.9678 & 73.51$\times$ \\
\bottomrule
\end{tabular}
}
\end{table}
\subsection{Ablation Experiment between Nora and Mano}
To verify that row-normalization indeed captures the row-diagonal dominance inherent in the Transformer Hessian, we conducted an ablation study. Given that Nora and Mano share similar algorithmic frameworks, it is crucial to distinguish their structural advantages from hyperparameter effects. While the previous experiments utilized Mano’s recommended $\text{weight decay}=0.1$ against Nora’s default $\text{weight decay}=0$, we eliminated this discrepancy in the ablation study by setting Mano’s $\text{weight decay}$ to $0$. As shown in Table \ref{xrsy}, the results demonstrate that Nora’s superiority stems from its intrinsic algorithmic architecture rather than an advantage gained from weight decay settings.
\begin{figure}[htbp]
    \centering
    \includegraphics[width=1.0\textwidth]{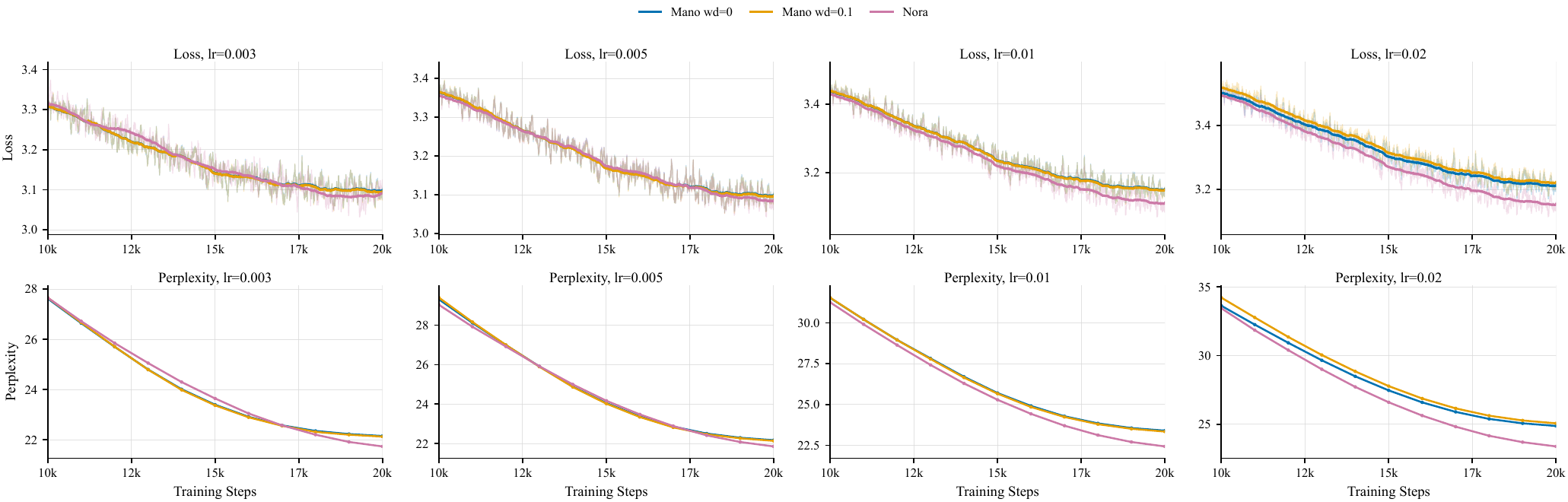}
    \caption{Training dynamics on the 135M model. This figure illustrates the perplexity and loss decay curves for Nora (default $\text{weight decay}=0$) in comparison with Mano (under $\text{weight decay}=0$ and $\text{weight decay}=0.1$).}
    \label{dbsy}
\end{figure}

Furthermore, Figure \ref{dbsy} illustrates the perplexity and loss decay curves for both Mano and Nora at the 135M model scale, under configurations of $\text{weight decay}=0$ and $\text{weight decay}=0.1$. It is evident that Nora consistently outperforms Mano across all settings. For a comprehensive overview, all remaining experimental results and corresponding convergence curves are provided in Appendix \ref{app:exp_details}, where we focus on the sensitivity of hyper-parameters and \underline{\textbf{stability}} of our proposed optimizer.
\
% \usepackage{multirow}
% \usepackage{multirow}

\section{Conclusion}
% This paper introduces the Nora optimizer, designed to satisfy the essential requirements of efficiency, stability, and speed. We demonstrate that Nora is inherently scalable and derive its corresponding learning rate scaling rules. Furthermore, we provide a rigorous convergence proof for Nora in non-convex settings. Empirical results confirm that Nora exhibits superior properties compared to existing state-of-the-art optimizers.
We introduce Nora, a normalized orthogonal row-alignment optimizer for scalable
LLM training. Nora is motivated by the geometry of scale-invariant neural
networks, where effective learning should mainly occur along angular rather than
radial directions. By projecting momentum onto the row-wise orthogonal complement
of the weights and applying row-wise normalization, Nora preserves stable
scale-invariant dynamics while retaining efficient Muon-like preconditioning
under the row-diagonal structure of Transformer Hessians. We further establish
Nora's scalability through $\mu$P-based width-scaling analysis and provide
non-convex convergence guarantees. Experiments on LLaMA-style models show that
Nora achieves the best validation loss and perplexity among the compared
matrix-based optimizers, while its row-wise normalization brings a clear speed
advantage over Newton--Schulz orthogonalization. Overall, Nora demonstrates that
efficiency, stability, and speed can be unified through a simple row-wise
geometric principle, offering a practical and principled direction for scalable
optimizer design.

\clearpage
\bibliographystyle{unsrt}
\normalem
\bibliography{example_paper.bib}

%%%%%%%%%%%%%%%%%%%%%%%%%%%%%%%%%%%%%%%%%%%%%%%%%%%%%%%%%%%%

\newpage
\tableofcontents
\newpage
\appendix
\section{Proof of Theorem}
\label{app:nora-proof}
\subsection{Proof of Theorem 4.1}
   \textbf{Theorem 4.1 (Nora Scaling under the Scaling Hypothesis)} \textit{Consider a neural network layer defined by $h = wx$, where the weights $w \in \mathbb{R}^{m \times n}$. Assume the input activation $x \in \mathbb{R}^n$ satisfies the scaling hypothesis under standard deep learning initialization: $\|x\|_2 \le \gamma \sqrt{n}$, where $\gamma = \Theta(1)$ is a constant. Suppose Nora updates the parameters by $w_{t+1} = w_t - \eta_t d_t$. Then $|\Delta h_i| \le \eta_t \gamma \sqrt{n}$. To achieve the stable feature learning limit required by $\mu$P theory—specifically, to ensure the activation update magnitude satisfies $|\Delta h_i| = \Theta(1)$, the learning rate $\eta_t$ for Nora must follow the width-scaling rule, $\eta_t \propto \frac{1}{\sqrt{n}}$.}
\begin{proof}
    For any output neuron $i$, the change in the hidden activation $\Delta h_i$ during the forward pass is:
    \begin{equation}
        \Delta h_i = (w_{t+1, i:} - w_{t, i:}) x = -\eta_t \langle d_{t,i:}, x \rangle
    \end{equation}
    By taking the absolute value and applying the Cauchy-Schwarz inequality, we obtain:
    \begin{equation}
        |\Delta h_i| = \eta_t |\langle (d_t)_{i:}, x \rangle| \le \eta_t \|(d_t)_{i:}\|_2 \|x\|_2
    \end{equation}
    Nora's row-wise normalization strictly ensures $\|(d_t)_{i:}\|_2 \le 1$. Combined with the input assumption $\|x\|_2 \le \gamma \sqrt{n}$, substituting these into the inequality yields a rigorous upper bound for the forward update:
    \begin{equation}
        |\Delta h_i| \le \eta_t \cdot 1 \cdot \gamma \sqrt{n} = \eta_t \gamma \sqrt{n}
    \end{equation}
    To ensure that the feature update magnitude $|\Delta h_i|$ neither diverges nor vanishes as $n \to \infty$ (i.e., satisfies the Maximal Update condition $|\Delta h_i| = \Theta(1)$), we require the upper bound $\eta_t \gamma \sqrt{n}$ to also be of order $\Theta(1)$. Given that $\gamma$ is a constant, this necessarily implies $\eta_t = \frac{c}{\sqrt{n}}$, where $c$ is a constant independent of $n$. This completes the proof.
\end{proof}
\subsection{Proof of Theorem 4.2}
   \textbf{Theorem 4.2 (symptotic Convergence))} \textit{Consider a hidden layer $h = wx$, where $w \in \mathbb{R}^{m \times n}$. Under the standard infinite-width random initialization hypothesis, assume the row vectors of $w$ follow $w_{i:} \sim \mathcal{N}(0, \frac{\sigma_w^2}{n} I_n)$, and the components of the input activation $x \in \mathbb{R}^n$ are independent with zero mean and variance $\sigma_x^2 = \Theta(1)$. Let the error signal backpropagated to this layer be $\delta_i = \Theta(1)$, and the vanilla gradient (without momentum) be $g_{i:} = \delta_i x^\top$. If Nora is applied to generate the update direction $d_{i:} = \text{RN}(\mathcal{P}_{w}^{r\perp}(g_{i:}))$, then as the network width $n \to \infty$, the inner product between the update direction and the input converges in probability to $\langle d_{i:}, x \rangle \xrightarrow{p} \text{sgn}(\delta_i) \sigma_x \sqrt{n}$. To achieve non-trivial feature learning such that $\Delta h_i = \Theta(1)$ as $n \to \infty$, the Nora learning rate must scale as $\eta = \frac{\eta_0}{\sqrt{n}}$.}
   \begin{proof}
       For output neuron $i$, Nora first removes the row-wise radial component of the gradient. Writing row vectors as columns inside inner products, this gives:
       \begin{equation}
           u_{i:} = \delta_i x - \frac{\delta_i \langle x, w_{i:} \rangle}{\|w_{i:}\|_2^2} w_{i:}.
       \end{equation}
       We now compute the high-dimensional orders. Since $x_j$ has variance $\sigma_x^2$ and $w_{i,j}$ has variance $\sigma_w^2/n$, the weak law of large numbers gives:
       \begin{equation}
           \|x\|_2^2 \xrightarrow{p} n \sigma_x^2 \implies \|x\|_2 = \sigma_x \sqrt{n} + o_p(\sqrt{n}).
       \end{equation}
       Similarly:
       \begin{equation}
         \|w_{i:}\|_2^2 \xrightarrow{p} \sigma_w^2 = \Theta(1).  
       \end{equation}
       Because $x$ and $w$ are independent and centered, $\langle x, w_{i:} \rangle$ is a sum of $n$ independent products with total variance $\sigma_x^2\sigma_w^2$. The central limit theorem gives:
       \begin{equation}
           \langle x, w_{i:} \rangle \xrightarrow{d} \mathcal{N}(0, \sigma_x^2 \sigma_w^2) = \mathcal{O}_p(1).
       \end{equation}
       Let $r_{i:}$ denote the radial component removed by the projection. Its Euclidean norm satisfies:
       \begin{equation}
           \|r_{i:}\|_2 = \left| \frac{\delta_i \langle x, w_{i:} \rangle}{\|w_{i:}\|_2^2} \right| \|w_{i:}\|_2 = \frac{|\delta_i| |\mathcal{O}_p(1)|}{\sigma_w^2} \sigma_w = \mathcal{O}_p(1).
       \end{equation}
       The primary gradient term $\delta_i x$ has norm $\Theta(\sqrt{n})$, while the removed radial component is $\mathcal{O}_p(1)$. Hence the projected vector satisfies:
       \begin{equation}
           \|u_{i:}\|_2 = \|\delta_i x - r_{i:}\|_2 = |\delta_i| \|x\|_2 + \mathcal{O}_p(1) = |\delta_i| \sigma_x \sqrt{n} (1 + o_p(1)).
       \end{equation}
       Row-wise normalization gives:
       \begin{equation}
           d_{i:} = \frac{u_{i:}}{\|u_{i:}\|_2} = \frac{\delta_i x - r_{i:}}{|\delta_i| \sigma_x \sqrt{n} (1 + o_p(1))}.
       \end{equation}
       The forward update depends on the inner product:
       \begin{equation}
          \langle d_{i:}, x \rangle = \frac{\delta_i \langle x, x \rangle - \langle r_{i:}, x \rangle}{|\delta_i| \sigma_x \sqrt{n} (1 + o_p(1))}.
       \end{equation}
       Since $\langle x, x \rangle = n \sigma_x^2 + o_p(n)$ and $|\langle r_{i:}, x \rangle| \le \|r_{i:}\|_2 \|x\|_2 = \mathcal{O}_p(\sqrt{n})$,
       \begin{equation}
           \langle d_{i:}, x \rangle = \frac{\delta_i n \sigma_x^2 - \mathcal{O}_p(\sqrt{n})}{|\delta_i| \sigma_x \sqrt{n} (1 + o_p(1))} = \text{sgn}(\delta_i) \sigma_x \sqrt{n} + o_p(\sqrt{n}).
       \end{equation}
       Thus $\langle d_{i:}, x \rangle \xrightarrow{p} \text{sgn}(\delta_i) \sigma_x \sqrt{n}$ holds. The actual forward update is $\Delta h_i = -\eta \langle d_{i:}, x \rangle = \mp \eta \sigma_x \sqrt{n}$. To achieve $\Delta h_i = \Theta(1)$ for non-trivial feature learning, it is necessary that $\eta = \Theta(n^{-1/2})$.
   \end{proof}

\subsection{Preliminaries for Theorem 4.7--Corollary 4.9}

\subsubsection{Setup}

For the convergence analysis, let \(m\) denote the number of rows in \(w\in\mathbb{R}^{m\times n}\). We reserve \(v_t\) for the momentum sequence. For any matrix \(w\) with nonzero rows, define the row-wise perpendicular projection:
\begin{equation}
    [\mathcal{P}^{r\perp}_{w}(x)]_{i:}
    :=
    x_{i:}
    -
    \frac{\langle x_{i:},w_{i:}\rangle}{\|w_{i:}\|_2^2}\,w_{i:}.
\end{equation}
This projection removes the row-wise radial component: \(\langle [\mathcal{P}^{r\perp}_{w}(x)]_{i:},w_{i:}\rangle=0\) for every row \(i\). Equivalently, a matrix \(z\) is row-wise perpendicular to \(w\) when \(\langle z_{i:},w_{i:}\rangle=0\) for all \(i\).

For \(t=0,\ldots,T-1\), Nora is indexed as:
\begin{equation}
 g_t=\nabla f(w_t;\xi_t),\qquad
 v_{t+1}=\beta v_t+(1-\beta)g_t,\qquad
 v_0=0,
\end{equation}
\begin{equation}
 v_{t+1}^{r\perp}:=\mathcal{P}^{r\perp}_{w_t}(v_{t+1}),\qquad
 d_t=\operatorname{RN}\bigl(v_{t+1}^{r\perp}\bigr),\qquad
 w_{t+1}=w_t-\eta d_t.
\end{equation}
Thus \(d_t\) is the row-wise normalization of the component of the momentum that is row-wise perpendicular to \(w_t\). We use the convention \(0/0=0\) in \(\operatorname{RN}\), applied row-wise:
\begin{equation}
[\operatorname{RN}(x)]_{i:}
=
\begin{cases}
\dfrac{x_{i:}}{\|x_{i:}\|_2}, & x_{i:}\neq 0,\\
0, & x_{i:}=0.
\end{cases}
\end{equation}
For a matrix \(x\in\mathbb{R}^{m\times n}\), define:
\begin{equation}
\|x\|_{1,2}:=\sum_{i=1}^{m}\|x_{i:}\|_2,
\qquad
\|x\|_{\infty,2}:=\max_{1\le i\le m}\|x_{i:}\|_2.
\end{equation}
We also define the natural filtration:
\begin{equation}
\mathcal{F}_t:=\sigma(w_0,\xi_0,\ldots,\xi_{t-1}),
\end{equation}
so that \(w_t\) is \(\mathcal{F}_t\)-measurable and \(g_t\) is sampled conditionally on \(\mathcal{F}_t\). Finally, define the projected gradient and momentum-tracking errors:
\begin{equation}
\mathcal{G}_t:=\mathcal{P}^{r\perp}_{w_t}\bigl(\nabla f(w_t)\bigr),
\qquad
 e_t:=v_{t+1}-\nabla f(w_t),
\qquad
\bar e_t:=v_{t+1}^{r\perp}-\mathcal{G}_t=\mathcal{P}^{r\perp}_{w_t}(e_t).
\end{equation}
The projected gradient \(\mathcal{G}_t\) is the relevant stationarity measure because Nora uses directions that are row-wise perpendicular to \(w_t\). Therefore radial components of \(\nabla f(w_t)\) do not contribute to the descent inner product with \(d_t\).

\subsubsection{Auxiliary Lemmas}

\begin{lemma}[Row-wise projection is non-expansive]
\label{lem:nora-proj}
For any \(w\in\mathbb{R}^{m\times n}\) with nonzero rows and any \(x\in\mathbb{R}^{m\times n}\),
\begin{equation}
\|\mathcal{P}^{r\perp}_{w}(x)\|_F\le \|x\|_F,
\qquad
\|\mathcal{P}^{r\perp}_{w}(x)\|_{1,2}\le \|x\|_{1,2},
\qquad
\|\mathcal{P}^{r\perp}_{w}(x)\|_{\infty,2}\le \|x\|_{\infty,2}.
\end{equation}
\end{lemma}

\begin{proof}
For each row \(i\), let
\begin{equation}
P_i:=I_n-\frac{w_{i:}^{\top}w_{i:}}{\|w_{i:}\|_2^2}.
\end{equation}
Then \([\mathcal{P}^{r\perp}_{w}(x)]_{i:}=x_{i:}P_i\). Since \(P_i\) is an orthogonal projector on \(\mathbb{R}^n\),
\begin{equation}
\|x_{i:}P_i\|_2\le \|x_{i:}\|_2.
\end{equation}
Summing over rows gives the Frobenius and \(\|\cdot\|_{1,2}\) bounds, while taking the maximum over rows gives the \(\|\cdot\|_{\infty,2}\) bound.
\end{proof}

\begin{lemma}[Basic geometry of Nora]
\label{lem:nora-geom}
Let \(z=\mathcal{P}^{r\perp}_{w}(x)\) and \(d=\operatorname{RN}(z)\). Then:
\begin{enumerate}
    \item \(\langle d_{i:},w_{i:}\rangle=0\) for every row \(i\);
    \item \(\|d\|_F\le \sqrt{m}\);
    \item \(\|d\|_{\infty,2}\le 1\);
    \item \(\langle z,d\rangle = \|z\|_{1,2}\);
    \item \(\langle z,d\rangle \ge \|z\|_F\).
\end{enumerate}
\end{lemma}

\begin{proof}
Each row \(z_{i:}\) is orthogonal to \(w_{i:}\). Every nonzero row of \(d\) is a scalar multiple of \(z_{i:}\), so it is also orthogonal to \(w_{i:}\). Thus \(d\) is row-wise perpendicular to \(w\).

For every row, \(\|d_{i:}\|_2\le 1\). Thus:
\begin{equation}
\|d\|_F^2=\sum_{i=1}^{m}\|d_{i:}\|_2^2\le m,
\qquad
\|d\|_{\infty,2}=\max_i \|d_{i:}\|_2\le 1.
\end{equation}
Moreover,
\begin{equation}
\langle z,d\rangle
=
\sum_{i=1}^{m}\left\langle z_{i:},\frac{z_{i:}}{\|z_{i:}\|_2}\right\rangle
=
\sum_{i=1}^{m}\|z_{i:}\|_2
=
\|z\|_{1,2}.
\end{equation}
Finally,
\begin{equation}
\|z\|_{1,2}=\sum_i \|z_{i:}\|_2
\ge
\left(\sum_i \|z_{i:}\|_2^2\right)^{1/2}
=
\|z\|_F,
\end{equation}
which gives \(\langle z,d\rangle\ge \|z\|_F\).
\end{proof}

\begin{lemma}[Descent under Frobenius smoothness]
\label{lem:nora-descent-f}
Under Assumption~\ref{ass:nora-smooth}(a), for every \(t\),
\begin{equation}
f(w_t)-f(w_{t+1})
\ge
\eta\langle \nabla f(w_t),d_t\rangle
-
\frac{L_F\eta^2 m}{2}.
\end{equation}
\end{lemma}

\begin{proof}
Let \(\Delta_t:=w_{t+1}-w_t\). By the fundamental theorem of calculus,
\begin{equation}
f(w_t+\Delta_t)-f(w_t)
=
\int_0^1 \langle \nabla f(w_t+s\Delta_t),\Delta_t\rangle\,ds.
\end{equation}
Subtracting \(\langle\nabla f(w_t),\Delta_t\rangle\) and applying Cauchy--Schwarz gives:
\begin{align}
f(w_t+\Delta_t)-f(w_t)-\langle\nabla f(w_t),\Delta_t\rangle
&=
\int_0^1 \langle \nabla f(w_t+s\Delta_t)-\nabla f(w_t),\Delta_t\rangle\,ds \\
&\le
\int_0^1
\|\nabla f(w_t+s\Delta_t)-\nabla f(w_t)\|_F\,\|\Delta_t\|_F\,ds .
\end{align}
Assumption~\ref{ass:nora-smooth}(a) implies:
\begin{equation}
\|\nabla f(w_t+s\Delta_t)-\nabla f(w_t)\|_F
\le
L_Fs\|\Delta_t\|_F.
\end{equation}
Hence the standard quadratic upper bound follows from the stated smoothness assumption:
\begin{equation}
f(w_t+\Delta_t)
\le
f(w_t)+\langle\nabla f(w_t),\Delta_t\rangle
+
\frac{L_F}{2}\|\Delta_t\|_F^2.
\end{equation}
Substituting \(\Delta_t=-\eta d_t\) gives:
\begin{equation}
f(w_{t+1})
\le
f(w_t)-\eta\langle \nabla f(w_t),d_t\rangle
+
\frac{L_F\eta^2}{2}\|d_t\|_F^2.
\end{equation}
Using Lemma~\ref{lem:nora-geom}(ii), \(\|d_t\|_F^2\le m\), which yields the claim.
\end{proof}

\begin{lemma}[Descent under matched \((\infty,2)\)-smoothness]
\label{lem:nora-descent-m}
Under Assumption~\ref{ass:nora-smooth}(b), for every \(t\),
\begin{equation}
f(w_t)-f(w_{t+1})
\ge
\eta\langle \nabla f(w_t),d_t\rangle
-
\frac{L_{\infty,2}\eta^2}{2}.
\end{equation}
\end{lemma}

\begin{proof}
Let \(\Delta_t:=w_{t+1}-w_t=-\eta d_t\). By the fundamental theorem of calculus,
\begin{equation}
f(w_t+\Delta_t)-f(w_t)
=
\langle \nabla f(w_t),\Delta_t\rangle
+
\int_0^1
\left\langle
\nabla f(w_t+s\Delta_t)-\nabla f(w_t),\Delta_t
\right\rangle ds.
\end{equation}
For any matrices \(A\) and \(B\), the \((1,2)\)--\((\infty,2)\) duality bound gives:
\begin{equation}
|\langle A,B\rangle|\le \|A\|_{1,2}\|B\|_{\infty,2}.
\end{equation}
Therefore, by Assumption~\ref{ass:nora-smooth}(b),
\begin{align}
\left|
\left\langle
\nabla f(w_t+s\Delta_t)-\nabla f(w_t),\Delta_t
\right\rangle
\right|
&\le
\|\nabla f(w_t+s\Delta_t)-\nabla f(w_t)\|_{1,2}\|\Delta_t\|_{\infty,2} \\
&\le
L_{\infty,2}s\|\Delta_t\|_{\infty,2}^2.
\end{align}
Lemma~\ref{lem:nora-geom}(iii) gives \(\|\Delta_t\|_{\infty,2}=\eta\|d_t\|_{\infty,2}\le\eta\). Hence:
\begin{equation}
f(w_{t+1})-f(w_t)
\le
-\eta\langle \nabla f(w_t),d_t\rangle
+
\int_0^1 L_{\infty,2}s\eta^2\,ds
=
-\eta\langle \nabla f(w_t),d_t\rangle
+
\frac{L_{\infty,2}\eta^2}{2}.
\end{equation}
Rearranging proves the claim.
\end{proof}

\begin{lemma}[Projected inner-product lower bound in Frobenius norm]
\label{lem:nora-ip-f}
For every \(t\),
\begin{equation}
\langle \nabla f(w_t),d_t\rangle
\ge
\|\mathcal{G}_t\|_F-(\sqrt{m}+1)\|e_t\|_F.
\end{equation}
\end{lemma}

\begin{proof}
Since \(d_t\) is row-wise perpendicular to \(w_t\) by Lemma~\ref{lem:nora-geom}(i), and \(\nabla f(w_t)-\mathcal{G}_t\) is row-wise parallel to \(w_t\), we have
\begin{equation}
\langle \nabla f(w_t),d_t\rangle
=
\langle \mathcal{G}_t,d_t\rangle.
\end{equation}
Since \(\mathcal{G}_t=v_{t+1}^{r\perp}-\bar e_t\),
\begin{equation}
\langle \mathcal{G}_t,d_t\rangle
=
\langle v_{t+1}^{r\perp},d_t\rangle-\langle \bar e_t,d_t\rangle.
\end{equation}
By Lemma~\ref{lem:nora-geom}(v),
\begin{equation}
\langle v_{t+1}^{r\perp},d_t\rangle\ge \|v_{t+1}^{r\perp}\|_F.
\end{equation}
By Cauchy--Schwarz and Lemma~\ref{lem:nora-geom}(ii),
\begin{equation}
|\langle \bar e_t,d_t\rangle|
\le
\|\bar e_t\|_F\|d_t\|_F
\le
\sqrt{m}\,\|\bar e_t\|_F.
\end{equation}
Also,
\begin{equation}
\|v_{t+1}^{r\perp}\|_F=\|\mathcal{G}_t+\bar e_t\|_F
\ge
\|\mathcal{G}_t\|_F-\|\bar e_t\|_F.
\end{equation}
Combining the above,
\begin{equation}
\langle \nabla f(w_t),d_t\rangle
\ge
\|\mathcal{G}_t\|_F-(\sqrt{m}+1)\|\bar e_t\|_F.
\end{equation}
Finally, by Lemma~\ref{lem:nora-proj},
\begin{equation}
\|\bar e_t\|_F=\|\mathcal{P}^{r\perp}_{w_t}(e_t)\|_F\le \|e_t\|_F.
\end{equation}
This proves the claim.
\end{proof}

\begin{lemma}[Projected inner-product lower bound in \((1,2)\)-norm]
\label{lem:nora-ip-12}
For every \(t\),
\begin{equation}
\langle \nabla f(w_t),d_t\rangle
\ge
\|\mathcal{G}_t\|_{1,2}-2\|e_t\|_{1,2}.
\end{equation}
\end{lemma}

\begin{proof}
As in Lemma~\ref{lem:nora-ip-f},
\begin{equation}
\langle \nabla f(w_t),d_t\rangle
=
\langle \mathcal{G}_t,d_t\rangle
=
\langle v_{t+1}^{r\perp},d_t\rangle-\langle \bar e_t,d_t\rangle.
\end{equation}
By Lemma~\ref{lem:nora-geom}(iv),
\begin{equation}
\langle v_{t+1}^{r\perp},d_t\rangle=\|v_{t+1}^{r\perp}\|_{1,2}.
\end{equation}
By duality and Lemma~\ref{lem:nora-geom}(iii),
\begin{equation}
|\langle \bar e_t,d_t\rangle|
\le
\|\bar e_t\|_{1,2}\|d_t\|_{\infty,2}
\le
\|\bar e_t\|_{1,2}.
\end{equation}
Also,
\begin{equation}
\|v_{t+1}^{r\perp}\|_{1,2}
=
\|\mathcal{G}_t+\bar e_t\|_{1,2}
\ge
\|\mathcal{G}_t\|_{1,2}-\|\bar e_t\|_{1,2}.
\end{equation}
Hence,
\begin{equation}
\langle \nabla f(w_t),d_t\rangle
\ge
\|\mathcal{G}_t\|_{1,2}-2\|\bar e_t\|_{1,2}.
\end{equation}
Using Lemma~\ref{lem:nora-proj},
\begin{equation}
\|\bar e_t\|_{1,2}\le \|e_t\|_{1,2},
\end{equation}
and the proof is complete.
\end{proof}

\begin{assumption}[Unbiased stochastic gradients]
\label{ass:nora-unbiased}
For all \(t\),
\begin{equation}
    \mathbb{E}[g_t\mid \mathcal{F}_t]=\nabla f(w_t).
\end{equation}
\end{assumption}

\begin{assumption}[Bounded variance]
\label{ass:nora-var}
There exists \(\sigma>0\) such that, for all \(t\),
\begin{equation}
\mathbb{E}\!\left[\|g_t-\nabla f(w_t)\|_F^2\mid \mathcal{F}_t\right]\le \frac{\sigma^2}{B},
\end{equation}
where \(B\) is the batch size.
\end{assumption}

\begin{assumption}[Lower bounded objective]
\label{ass:nora-lower}
The objective is bounded below by \(f^\star\). We write \(\Delta:=f(w_0)-f^\star\).
\end{assumption}

\begin{lemma}[Momentum tracking under Frobenius smoothness]
\label{lem:nora-track-f}
Under Assumptions~\ref{ass:nora-smooth}(a), \ref{ass:nora-unbiased}, and \ref{ass:nora-var},
\begin{equation}
\sum_{t=0}^{T-1}\mathbb{E}\bigl[\|e_t\|_F\bigr]
\le
(T-1)\frac{L_F\eta\sqrt{m}\beta}{1-\beta}
+
T\frac{\sigma}{\sqrt{B}}\sqrt{\frac{1-\beta}{1+\beta}}.
\end{equation}
\end{lemma}

\begin{proof}
Let \(\zeta_t:=g_t-\nabla f(w_t)\). Then
\begin{equation}
\mathbb{E}[\zeta_t\mid \mathcal{F}_t]=0
\end{equation}
by Assumption~\ref{ass:nora-unbiased}, and
\begin{equation}
\mathbb{E}\bigl[\|\zeta_t\|_F^2\mid \mathcal{F}_t\bigr]\le \frac{\sigma^2}{B}
\end{equation}
by Assumption~\ref{ass:nora-var}.

For \(t\ge 1\),
\begin{equation}
\begin{aligned}
e_t
&=
v_{t+1}-\nabla f(w_t)\\
&=
\beta v_t+(1-\beta)g_t-\nabla f(w_t)\\
&=
\beta\bigl(v_t-\nabla f(w_{t-1})\bigr)
+\beta\bigl(\nabla f(w_{t-1})-\nabla f(w_t)\bigr)
+(1-\beta)\zeta_t\\
&=
\beta e_{t-1}
+\beta\bigl(\nabla f(w_{t-1})-\nabla f(w_t)\bigr)
+(1-\beta)\zeta_t,
\end{aligned}
\end{equation}
while \(e_0=(1-\beta)\zeta_0\). Unrolling the recursion gives
\begin{equation}
e_t
=
\sum_{j=0}^{t}\beta^{t-j}(1-\beta)\zeta_j
+
\sum_{j=1}^{t}\beta^{t-j+1}\bigl(\nabla f(w_{j-1})-\nabla f(w_j)\bigr).
\end{equation}
Therefore,
\begin{align}
\|e_t\|_F
&\le
\left\|
\sum_{j=0}^{t}\beta^{t-j}(1-\beta)\zeta_j
\right\|_F
+
\sum_{j=1}^{t}\beta^{t-j+1}\|\nabla f(w_{j-1})-\nabla f(w_j)\|_F.
\end{align}
Using Assumption~\ref{ass:nora-smooth}(a), \(w_j-w_{j-1}=-\eta d_{j-1}\), and Lemma~\ref{lem:nora-geom}(ii),
\begin{equation}
\begin{aligned}
\sum_{j=1}^{t}\beta^{t-j+1}\|\nabla f(w_{j-1})-\nabla f(w_j)\|_F
&\le
\sum_{j=1}^{t}\beta^{t-j+1}L_F\|w_{j-1}-w_j\|_F\\
&=
\sum_{j=1}^{t}\beta^{t-j+1}L_F\eta\|d_{j-1}\|_F\\
&\le
L_F\eta\sqrt{m}\sum_{k=1}^{t}\beta^k\\
&\le
L_F\eta\sqrt{m}\frac{\beta}{1-\beta}.
\end{aligned}
\end{equation}

For the noise term, Jensen's inequality implies
\begin{equation}
\mathbb{E}\left\|
\sum_{j=0}^{t}\beta^{t-j}(1-\beta)\zeta_j
\right\|_F
\le
\sqrt{
\mathbb{E}\left\|
\sum_{j=0}^{t}\beta^{t-j}(1-\beta)\zeta_j
\right\|_F^2 }.
\end{equation}
Since \(\{\zeta_t\}\) is a martingale difference sequence, the cross terms vanish, and thus
\begin{equation}
\begin{aligned}
\mathbb{E}\left\|
\sum_{j=0}^{t}\beta^{t-j}(1-\beta)\zeta_j
\right\|_F^2
&=
\sum_{j=0}^{t}\beta^{2(t-j)}(1-\beta)^2\mathbb{E}\|\zeta_j\|_F^2\\
&\le
\frac{\sigma^2}{B}(1-\beta)^2\sum_{k=0}^{t}\beta^{2k}\\
&\le
\frac{\sigma^2}{B}\frac{1-\beta}{1+\beta}.
\end{aligned}
\end{equation}
Hence,
\begin{equation}
\mathbb{E}\left\|
\sum_{j=0}^{t}\beta^{t-j}(1-\beta)\zeta_j
\right\|_F
\le
\frac{\sigma}{\sqrt{B}}\sqrt{\frac{1-\beta}{1+\beta}}.
\end{equation}
Combining the previous bounds and summing over \(t=0,\ldots,T-1\) yields
\begin{equation}
\sum_{t=0}^{T-1}\mathbb{E}\bigl[\|e_t\|_F\bigr]
\le
(T-1)\frac{L_F\eta\sqrt{m}\beta}{1-\beta}
+
T\frac{\sigma}{\sqrt{B}}\sqrt{\frac{1-\beta}{1+\beta}}.
\end{equation}
\end{proof}

\begin{lemma}[Momentum tracking under matched \((\infty,2)\)-smoothness]
\label{lem:nora-track-m}
Under Assumptions~\ref{ass:nora-smooth}(b), \ref{ass:nora-unbiased}, and \ref{ass:nora-var},
\begin{equation}
\sum_{t=0}^{T-1}\mathbb{E}\bigl[\|e_t\|_{1,2}\bigr]
\le
(T-1)\frac{L_{\infty,2}\eta\beta}{1-\beta}
+
T\frac{\sqrt{m}\sigma}{\sqrt{B}}\sqrt{\frac{1-\beta}{1+\beta}}.
\end{equation}
\end{lemma}

\begin{proof}
The recursion for \(e_t\) is the same as in Lemma~\ref{lem:nora-track-f}. Hence,
\begin{equation}
e_t
=
\sum_{j=0}^{t}\beta^{t-j}(1-\beta)\zeta_j
+
\sum_{j=1}^{t}\beta^{t-j+1}\bigl(\nabla f(w_{j-1})-\nabla f(w_j)\bigr).
\end{equation}
By the triangle inequality,
\begin{equation}
\|e_t\|_{1,2}
\le
\left\|
\sum_{j=0}^{t}\beta^{t-j}(1-\beta)\zeta_j
\right\|_{1,2}
+
\sum_{j=1}^{t}\beta^{t-j+1}\|\nabla f(w_{j-1})-\nabla f(w_j)\|_{1,2}.
\end{equation}
Using Assumption~\ref{ass:nora-smooth}(b), \(w_j-w_{j-1}=-\eta d_{j-1}\), and Lemma~\ref{lem:nora-geom}(iii),
\begin{equation}
\begin{aligned}
\sum_{j=1}^{t}\beta^{t-j+1}\|\nabla f(w_{j-1})-\nabla f(w_j)\|_{1,2}
&\le
\sum_{j=1}^{t}\beta^{t-j+1}L_{\infty,2}\|w_{j-1}-w_j\|_{\infty,2}\\
&=
\sum_{j=1}^{t}\beta^{t-j+1}L_{\infty,2}\eta\|d_{j-1}\|_{\infty,2}\\
&\le
L_{\infty,2}\eta\sum_{k=1}^{t}\beta^k\\
&\le
L_{\infty,2}\eta\frac{\beta}{1-\beta}.
\end{aligned}
\end{equation}
For the noise term, we use \(\|A\|_{1,2}\le \sqrt{m}\|A\|_F\), hence
\begin{equation}
\mathbb{E}\left\|
\sum_{j=0}^{t}\beta^{t-j}(1-\beta)\zeta_j
\right\|_{1,2}
\le
\sqrt{m}\,
\mathbb{E}\left\|
\sum_{j=0}^{t}\beta^{t-j}(1-\beta)\zeta_j
\right\|_F.
\end{equation}
Applying the Frobenius-noise bound from the proof of Lemma~\ref{lem:nora-track-f},
\begin{equation}
\mathbb{E}\left\|
\sum_{j=0}^{t}\beta^{t-j}(1-\beta)\zeta_j
\right\|_{1,2}
\le
\frac{\sqrt{m}\sigma}{\sqrt{B}}\sqrt{\frac{1-\beta}{1+\beta}}.
\end{equation}
Summing over \(t=0,\ldots,T-1\) proves the lemma.
\end{proof}

\subsection{Proof of Theorem 4.7}

\textbf{Theorem 4.7.} {\itshape
Suppose Assumptions~\ref{ass:nora-smooth}(b), \ref{ass:nora-unbiased}, \ref{ass:nora-var}, and \ref{ass:nora-lower} hold, and Nora uses a constant step size \(\eta_t=\eta\). Then:
\begin{equation}
\frac{1}{T}\sum_{t=0}^{T-1}\mathbb{E}\bigl[\|\mathcal{G}_t\|_{1,2}\bigr]
\le
\frac{\Delta}{T\eta}
+
2\left[
\left(1-\frac{1}{T}\right)\frac{L_{\infty,2}\eta\beta}{1-\beta}
+
\frac{\sqrt{m}\sigma}{\sqrt{B}}
\sqrt{\frac{1-\beta}{1+\beta}}
\right]
+
\frac{L_{\infty,2}\eta}{2}.
\end{equation}
With the parameter choice:
\begin{equation}
\eta=\sqrt{\frac{(1-\beta)\Delta}{L_{\infty,2}T}},
\qquad
1-\beta=
\min\left\{
\frac{\sqrt{L_{\infty,2}\Delta}}{2\sqrt{m}\sigma\sqrt{T}},
\,1
\right\},
\end{equation}
Nora reaches an \(\epsilon\)-stationary point in the projected \(\|\cdot\|_{1,2}\) sense with complexity:
\begin{equation}
    T=\mathcal{O}\!\left(mL_{\infty,2}\sigma^2\Delta\,\epsilon^{-4}\right).
\end{equation}
}

\begin{proof}
Summing Lemma~\ref{lem:nora-descent-m} over \(t=0,\ldots,T-1\), we obtain
\begin{equation}
f(w_0)-f(w_T)
\ge
\eta\sum_{t=0}^{T-1}\langle \nabla f(w_t),d_t\rangle
-
\frac{TL_{\infty,2}\eta^2}{2}.
\end{equation}
By Assumption~\ref{ass:nora-lower},
\begin{equation}
\Delta
\ge
\eta\sum_{t=0}^{T-1}\langle \nabla f(w_t),d_t\rangle
-
\frac{TL_{\infty,2}\eta^2}{2}.
\end{equation}
Applying Lemma~\ref{lem:nora-ip-12},
\begin{equation}
\Delta
\ge
\eta\sum_{t=0}^{T-1}\|\mathcal{G}_t\|_{1,2}
-
2\eta\sum_{t=0}^{T-1}\|e_t\|_{1,2}
-
\frac{TL_{\infty,2}\eta^2}{2}.
\end{equation}
Taking expectations and using Lemma~\ref{lem:nora-track-m},
\begin{equation}
\begin{aligned}
\eta\sum_{t=0}^{T-1}\mathbb{E}\bigl[\|\mathcal{G}_t\|_{1,2}\bigr]
&\le
\Delta
+
2\eta\sum_{t=0}^{T-1}\mathbb{E}\bigl[\|e_t\|_{1,2}\bigr]
+
\frac{TL_{\infty,2}\eta^2}{2}\\
&\le
\Delta
+
2\eta
\left[
(T-1)\frac{L_{\infty,2}\eta\beta}{1-\beta}
+
T\frac{\sqrt{m}\sigma}{\sqrt{B}}\sqrt{\frac{1-\beta}{1+\beta}}
\right]
+
\frac{TL_{\infty,2}\eta^2}{2}.
\end{aligned}
\end{equation}
Dividing both sides by \(T\eta\) yields
\begin{equation}
\frac{1}{T}\sum_{t=0}^{T-1}\mathbb{E}\bigl[\|\mathcal{G}_t\|_{1,2}\bigr]
\le
\frac{\Delta}{T\eta}
+
2\left[
\left(1-\frac{1}{T}\right)\frac{L_{\infty,2}\eta\beta}{1-\beta}
+
\frac{\sqrt{m}\sigma}{\sqrt{B}}\sqrt{\frac{1-\beta}{1+\beta}}
\right]
+
\frac{L_{\infty,2}\eta}{2}.
\end{equation}
This proves the explicit bound.

To derive the rate, set \(q:=1-\beta\).  For \(B=1\), the preceding
bound implies, up to universal constants,
\begin{equation}
R_T
\lesssim
\frac{\Delta}{T\eta}
+
\frac{L_{\infty,2}\eta}{q}
+
\sqrt m\sigma\sqrt q
+
L_{\infty,2}\eta .
\end{equation}
Since \(q\le 1\), the last term is dominated by
\(L_{\infty,2}\eta/q\).  For a fixed \(q\), we balance the descent term
and the momentum-tracking term:
\begin{equation}
\frac{\Delta}{T\eta}
\asymp
\frac{L_{\infty,2}\eta}{q},
\qquad
\text{which gives}
\qquad
\eta=
\sqrt{\frac{q\Delta}{L_{\infty,2}T}} .
\end{equation}
With this choice, the bound reduces to
\begin{equation}
R_T
\lesssim
\sqrt{\frac{L_{\infty,2}\Delta}{qT}}
+
\sqrt m\sigma\sqrt q .
\end{equation}
The first term increases as \(q\) decreases, whereas the stochastic-noise
term decreases.  Balancing them gives
\(q\asymp \sqrt{L_{\infty,2}\Delta}/(\sqrt m\sigma\sqrt T)\).
Thus we take
\begin{equation}
q=1-\beta
=
\min\left\{
\frac{\sqrt{L_{\infty,2}\Delta}}
{2\sqrt m\sigma\sqrt T},
1
\right\},
\qquad
\eta=
\sqrt{\frac{(1-\beta)\Delta}{L_{\infty,2}T}} .
\end{equation}
This is a constructive parameter choice used to obtain the complexity
bound; it is not meant to be a unique tuning rule for practical training.
\end{proof}

\subsection{Proof of Proposition 4.8}

\textbf{Proposition 4.8.} {\itshape
Suppose Assumptions~\ref{ass:nora-smooth}(a), \ref{ass:nora-unbiased}, \ref{ass:nora-var}, and \ref{ass:nora-lower} hold, and Nora uses a constant step size \(\eta_t=\eta\). Then:
\begin{equation}
\frac{1}{T}\sum_{t=0}^{T-1}\mathbb{E}\bigl[\|\mathcal{G}_t\|_F\bigr]
\le
\frac{\Delta}{T\eta}
+
(\sqrt{m}+1)\left[
\left(1-\frac{1}{T}\right)\frac{L_F\eta\sqrt{m}\beta}{1-\beta}
+
\frac{\sigma}{\sqrt{B}}
\sqrt{\frac{1-\beta}{1+\beta}}
\right]
+
\frac{L_F\eta m}{2},
\end{equation}
and similarly,
\begin{equation}
\frac{1}{T}\sum_{t=0}^{T-1}\mathbb{E}\bigl[\|\mathcal{G}_t\|_{1,2}\bigr]
\le
\frac{\Delta}{T\eta}
+
2\left[
\left(1-\frac{1}{T}\right)\frac{L_F\eta m\beta}{1-\beta}
+
\frac{\sqrt{m}\sigma}{\sqrt{B}}
\sqrt{\frac{1-\beta}{1+\beta}}
\right]
+
\frac{L_F\eta m}{2}.
\end{equation}
Consequently, Nora reaches an \(\epsilon\)-stationary point in either projected measure with complexity:
\begin{equation}
T=\mathcal{O}\!\left(m^2L_F\sigma^2\Delta\,\epsilon^{-4}\right).
\end{equation}
}

\begin{proof}
We first prove the Frobenius-norm bound. Summing Lemma~\ref{lem:nora-descent-f} over \(t=0,\ldots,T-1\) gives
\begin{equation}
f(w_0)-f(w_T)
\ge
\eta\sum_{t=0}^{T-1}\langle \nabla f(w_t),d_t\rangle
-
\frac{TL_F\eta^2m}{2}.
\end{equation}
By Assumption~\ref{ass:nora-lower}, \(f(w_T)\ge f^\star\), hence
\begin{equation}
\Delta
\ge
\eta\sum_{t=0}^{T-1}\langle \nabla f(w_t),d_t\rangle
-
\frac{TL_F\eta^2m}{2}.
\end{equation}
Applying Lemma~\ref{lem:nora-ip-f},
\begin{equation}
\Delta
\ge
\eta\sum_{t=0}^{T-1}\|\mathcal{G}_t\|_F
-
\eta(\sqrt{m}+1)\sum_{t=0}^{T-1}\|e_t\|_F
-
\frac{TL_F\eta^2m}{2}.
\end{equation}
Taking expectation and using Lemma~\ref{lem:nora-track-f},
\begin{equation}
\begin{aligned}
\eta\sum_{t=0}^{T-1}\mathbb{E}\bigl[\|\mathcal{G}_t\|_F\bigr]
&\le
\Delta
+
\eta(\sqrt{m}+1)\sum_{t=0}^{T-1}\mathbb{E}\bigl[\|e_t\|_F\bigr]
+
\frac{TL_F\eta^2m}{2}\\
&\le
\Delta
+
\eta(\sqrt{m}+1)
\left[
(T-1)\frac{L_F\eta\sqrt{m}\beta}{1-\beta}
+
T\frac{\sigma}{\sqrt{B}}\sqrt{\frac{1-\beta}{1+\beta}}
\right]
+
\frac{TL_F\eta^2m}{2}.
\end{aligned}
\end{equation}
Dividing both sides by \(T\eta\) yields
\begin{equation}
\frac{1}{T}\sum_{t=0}^{T-1}\mathbb{E}\bigl[\|\mathcal{G}_t\|_F\bigr]
\le
\frac{\Delta}{T\eta}
+
(\sqrt{m}+1)\left[
\left(1-\frac{1}{T}\right)\frac{L_F\eta\sqrt{m}\beta}{1-\beta}
+
\frac{\sigma}{\sqrt{B}}\sqrt{\frac{1-\beta}{1+\beta}}
\right]
+
\frac{L_F\eta m}{2}.
\end{equation}

We next prove the \(\|\cdot\|_{1,2}\) bound under the same Frobenius smoothness assumption. Starting again from Lemma~\ref{lem:nora-descent-f},
\begin{equation}
\Delta
\ge
\eta\sum_{t=0}^{T-1}\langle \nabla f(w_t),d_t\rangle
-
\frac{TL_F\eta^2m}{2}.
\end{equation}
Using Lemma~\ref{lem:nora-ip-12},
\begin{equation}
\Delta
\ge
\eta\sum_{t=0}^{T-1}\|\mathcal{G}_t\|_{1,2}
-
2\eta\sum_{t=0}^{T-1}\|e_t\|_{1,2}
-
\frac{TL_F\eta^2m}{2}.
\end{equation}
Since \(\|A\|_{1,2}\le \sqrt{m}\|A\|_F\), Lemma~\ref{lem:nora-track-f} implies
\begin{equation}
\sum_{t=0}^{T-1}\mathbb{E}\bigl[\|e_t\|_{1,2}\bigr]
\le
\sqrt{m}\sum_{t=0}^{T-1}\mathbb{E}\bigl[\|e_t\|_F\bigr]
\le
(T-1)\frac{L_F\eta m\beta}{1-\beta}
+
T\frac{\sqrt{m}\sigma}{\sqrt{B}}\sqrt{\frac{1-\beta}{1+\beta}}.
\end{equation}
Substituting the above bound and dividing by \(T\eta\) gives
\begin{equation}
\frac{1}{T}\sum_{t=0}^{T-1}\mathbb{E}\bigl[\|\mathcal{G}_t\|_{1,2}\bigr]
\le
\frac{\Delta}{T\eta}
+
2\left[
\left(1-\frac{1}{T}\right)\frac{L_F\eta m\beta}{1-\beta}
+
\frac{\sqrt{m}\sigma}{\sqrt{B}}\sqrt{\frac{1-\beta}{1+\beta}}
\right]
+
\frac{L_F\eta m}{2}.
\end{equation}

It remains to justify the stated complexity. Take \(B=1\) and choose:
\begin{equation}
\eta=\sqrt{\frac{(1-\beta)\Delta}{L_FmT}},
\qquad
1-\beta=
\min\left\{
\frac{\sqrt{L_F\Delta}}{2\sigma\sqrt{T}},
\,1
\right\}.
\end{equation}
Substituting this choice into either of the two preceding bounds gives
\begin{equation}
\frac{1}{T}\sum_{t=0}^{T-1}\mathbb{E}\bigl[\|\mathcal{G}_t\|_F\bigr]
=
\mathcal{O}\!\left(
\sqrt[4]{\frac{m^2L_F\sigma^2\Delta}{T}}
+
\sqrt{\frac{mL_F\Delta}{T}}
\right),
\end{equation}
and similarly,
\begin{equation}
\frac{1}{T}\sum_{t=0}^{T-1}\mathbb{E}\bigl[\|\mathcal{G}_t\|_{1,2}\bigr]
=
\mathcal{O}\!\left(
\sqrt[4]{\frac{m^2L_F\sigma^2\Delta}{T}}
+
\sqrt{\frac{mL_F\Delta}{T}}
\right).
\end{equation}
Thus Nora reaches an \(\epsilon\)-stationary point in either projected measure with iteration complexity:
\begin{equation}
T=\mathcal{O}\!\left(m^2L_F\sigma^2\Delta\,\epsilon^{-4}\right).
\end{equation}
\end{proof}

\subsection{Proof of Corollary 4.9}

\textbf{Corollary 4.9.} {\itshape
Assume, in addition, that \(f\) is row-wise scale invariant, namely:
\begin{equation}
f(Dw)=f(w),\qquad \forall w\in\mathbb{R}^{m\times n},\ \forall D\succ 0\ \text{diagonal}.
\end{equation}
Then \(\nabla f(w_t)\) is row-wise perpendicular to \(w_t\) for all \(t\), i.e.,
\begin{equation}
\mathcal{P}^{r\perp}_{w_t}\bigl(\nabla f(w_t)\bigr)=\nabla f(w_t).
\end{equation}
Therefore, Theorem~\ref{thm:nora-main} and Proposition~\ref{prop:nora-frob} hold with \(\mathcal{G}_t\) replaced by \(\nabla f(w_t)\).
}

\begin{proof}
Fix any row index \(i\), and define the diagonal matrix
\begin{equation}
D_i(c):=I+(c-1)e_ie_i^\top,\qquad c>0.
\end{equation}
By row-wise scale invariance,
\begin{equation}
f(D_i(c)w)=f(w),\qquad \forall c>0.
\end{equation}
Differentiating both sides with respect to \(c\) at \(c=1\), we obtain
\begin{equation}
0
=
\left.\frac{d}{dc}f(D_i(c)w)\right|_{c=1}
=
\left\langle \nabla f(w)_{i:},\, w_{i:}\right\rangle.
\end{equation}
Since \(i\) is arbitrary, every row of \(\nabla f(w)\) is orthogonal to the corresponding row of \(w\). Hence
\begin{equation}
\mathcal{P}^{r\perp}_{w}\bigl(\nabla f(w)\bigr)=\nabla f(w).
\end{equation}
Applying this identity to every iterate \(w_t\) shows that
\begin{equation}
\mathcal{G}_t=\nabla f(w_t),\qquad \forall t.
\end{equation}
Therefore, the statements of Theorem~\ref{thm:nora-main} and Proposition~\ref{prop:nora-frob} hold with \(\mathcal{G}_t\) replaced by \(\nabla f(w_t)\).
\end{proof}

\newpage
\section{Additional Experimental Details}
\label{app:exp_details}

\subsection{Model and Training Configurations}

Table~\ref{tab:app_training_config} reports the detailed model and training configurations used in the 60M and 135M experiments. Both settings use context length 256, global batch size 512, cosine learning-rate decay, bf16 precision, gradient clipping at 1.0, evaluation every 1,000 steps, and checkpointing every 5,000 steps.

\begin{table*}[t]
\centering
\caption{Detailed model and training configurations.}
\label{tab:app_training_config}
\resizebox{0.7\linewidth}{!}{
\begin{tabular}{ccc}
\toprule
Configuration & 60M & 135M \\
\midrule
Layers & 8 & 12 \\
Hidden size & 512 & 768 \\
Attention heads & 8 & 12 \\
MLP size & 1376 & 2048 \\
Maximum context length & 256 & 256 \\
Training steps & 10,000 & 20,000 \\
GPUs & 2 & 4 \\
Per-GPU batch size & 64 & 64 \\
Gradient accumulation & 4 & 2 \\
Global batch size & 512 & 512 \\
Nominal token budget & 1.31B & 2.62B \\
Base learning rate & $1\times 10^{-3}$ & $1\times 10^{-3}$ \\
Auxiliary Adam learning rate & $5\times 10^{-3}$ & $3\times 10^{-3}$ \\
Warmup steps & 1,000 & 2,000 \\
Learning-rate schedule & cosine & cosine \\
Gradient clipping & 1.0 & 1.0 \\
Precision & bf16 & bf16 \\
Evaluation frequency & every 1,000 steps & every 1,000 steps \\
Checkpoint frequency & every 5,000 steps & every 5,000 steps \\
\bottomrule
\end{tabular}
}
\end{table*}

\subsection{Matrix Learning-rate Sweep Grids}

Table~\ref{tab:app_matrix_lr_sweep} lists the matrix learning-rate sweep grids for all optimizers. Nora uses the same matrix learning-rate grid as Mano and RMNP, but differs in weight decay: Nora uses weight decay $0$, while Muon, Mano, and RMNP use weight decay $0.1$.

\begin{table*}[t]
\centering
\caption{Matrix learning-rate sweep grids.}
\label{tab:app_matrix_lr_sweep}
\resizebox{0.7\linewidth}{!}{
\begin{tabular}{cccc}
\toprule
Model size & Optimizer & Weight decay & Matrix learning-rate grid \\
\midrule
60M & Muon & 0.1 & $\{0.005, 0.01, 0.02, 0.03, 0.04\}$ \\
60M & Mano & 0.1 & $\{0.001, 0.004, 0.005, 0.01, 0.02\}$ \\
60M & RMNP & 0.1 & $\{0.001, 0.004, 0.005, 0.01, 0.02\}$ \\
60M & Nora & 0.0 & $\{0.001, 0.004, 0.005, 0.01, 0.02\}$ \\
\midrule
135M & Muon & 0.1 & $\{0.005, 0.01, 0.02, 0.03\}$ \\
135M & Mano & 0.1 & $\{0.003, 0.005, 0.01, 0.02\}$ \\
135M & RMNP & 0.1 & $\{0.003, 0.005, 0.01, 0.02\}$ \\
135M & Nora & 0.0 & $\{0.003, 0.005, 0.01, 0.02\}$ \\
\bottomrule
\end{tabular}
}
\end{table*}
\subsection{Other Experimental Results}
In Tables \ref{t7}, \ref{t8}, \ref{t9}, and \ref{t10}, we present the perplexity and loss results across 60M and 135M model scales, evaluated under various algorithms and learning rate configurations. As illustrated, Nora demonstrates a significant advantage over all compared baselines across these diverse settings.
\begin{table}[htbp]
\centering
\caption{Perplexity results of all baselines under different learning rates. (60M)}
\begin{tabular}{c|c|c|c|c}
\toprule
Perplexity & 0.001 & 0.004 & 0.01  & 0.02  \\
\midrule
Mano       & 34.09 & 29.55 & 31.53 & 34.39 \\
\midrule
RMNP       & 35.42 & 30.12 & 30.30 & 93.26 \\
\midrule
Muon       & 41.65 & 31.44 & 31.09 & 31.99 \\
\midrule
Nora       & 31.24 & 28.94 & 29.89 & 31.17 \\
\bottomrule
\end{tabular}
\label{t7}
\end{table}

\begin{table}[h]
\centering
\caption{Loss results of all baselines under different learning rates. (60M)}
\begin{tabular}{c|c|c|c|c}
\toprule
Loss & 0.001 & 0.004 & 0.01  & 0.02  \\
\midrule
Mano & 3.529 & 3.386 & 3.451 & 3.538 \\
\midrule
RMNP & 3.567 & 3.405 & 3.411 & 4.535 \\
\midrule
Muon & 3.729 & 3.448 & 3.437 & 3.466 \\
\midrule
Nora & 3.442 & 3.365 & 3.398 & 3.440 \\
\bottomrule
\end{tabular}
\label{t8}
\end{table}

\begin{table}[h]
\centering
\caption{Perplexity results of all baselines under different learning rates. (135M)}
\begin{tabular}{c|c|c|c|c}
\toprule
Matrix LR & 0.003 & 0.005 & 0.01 & 0.02 \\
\midrule
Mano & 22.13 & 22.14 & 23.34 & 25.07 \\
Nora & 21.74 & 21.86 & 22.43 & 23.39 \\
\midrule
Matrix LR & 0.005 & 0.01 & 0.02 & 0.03 \\
\midrule
RMNP & 22.62 & 22.46 & 22.64 & 22.67 \\
Muon & 23.40 & 23.17 & 24.06 & 23.23 \\
\bottomrule
\end{tabular}
\label{t9}
\end{table}

\begin{table}[h]
\centering
\caption{Loss results of all baselines under different learning rates. (135M)}
\begin{tabular}{c|c|c|c|c}
\toprule
Matrix LR & 0.003 & 0.005 & 0.01 & 0.02 \\
\midrule
Mano & 3.097 & 3.098 & 3.150 & 3.222 \\
Nora & 3.079 & 3.085 & 3.111 & 3.152 \\
\midrule
Matrix LR & 0.005 & 0.01 & 0.02 & 0.03 \\
\midrule
RMNP & 3.119 & 3.112 & 3.120 & 3.121 \\
Muon & 3.153 & 3.143 & 3.180 & 3.146 \\
\bottomrule
\end{tabular}
\label{t10}
\end{table}

\clearpage
\section{Running and Reference Code}
\label{code}
\subsection{Run Quickly}
We provide a dedicated repository containing the Nora\footnote{https://github.com/Yuan-Jinghui/Nora} source code, along with detailed rationales and instructions for replacing Adam with Nora. Additionally, a separate reproduction repository\footnote{https://github.com/JiaxuanZou0714/Lrp} is made available to facilitate the rapid replication of all experimental results presented in this paper.

Please set up the C4 \cite{c4} dataset yourself and download the reproduction repository, then run:
\begin{lstlisting}[language=Python]
SCRIPT_DIR="$(cd "$(dirname "${BASH_SOURCE[0]}")" && pwd)"

exec "$SCRIPT_DIR/train_universal.sh" \
    --model_size 135m \
    --optimizer nora \
    --num_gpus 4 \
    --lr_matrix 0.005 \
    --lr_adam 0.02 \
    --num_steps 20000 \
    --batch_size 64 \
    --total_batch_size 512 \
    --warmup_steps 2000 \
    --weight_decay 0.0 \
    --save_every 10000 \
    --eval_every 1000 \
    "$@"
\end{lstlisting}
It should be noted that the results for Muon, RMNP, and Mano can be reproduced using the exact same experimental setup as Nora.

\subsection{Reference Code}
\begin{lstlisting}[language=Python]
import math

import torch
import torch.nn.functional as F


LOW_PRECISION_DTYPES = (torch.float16, torch.bfloat16)


class Nora(torch.optim.Optimizer):
    """Normalized Orthogonal Row Alignment optimizer for scalable matrix training."""

    def __init__(
        self,
        param_groups,
        lr_nora=0.005,
        lr_adam=0.001,
        momentum=0.95,
        beta=0.95,
        weight_decay=0.0,
        betas=(0.9, 0.95),
        eps=1e-10,
    ):
        defaults = dict(
            lr_nora=lr_nora,
            lr_adam=lr_adam,
            momentum=momentum,
            beta=beta,
            weight_decay=weight_decay,
            betas=betas,
            eps=eps,
        )
        super().__init__(param_groups, defaults)

    def step(self, closure=None):
        loss = None
        if closure is not None:
            loss = closure()

        for group in self.param_groups:
            lr = group["lr"]
            momentum = group.get("momentum", 0.95)
            beta = group.get("beta", 0.95)
            weight_decay = group.get("weight_decay", 0.0)
            betas = group.get("betas", (0.9, 0.95))
            eps = group.get("eps", 1e-10)
            is_nora = group.get("is_nora", True)

            for p in group["params"]:
                if p.grad is None:
                    continue

                grad = p.grad.data
                param_state = self.state.setdefault(p, {})

                use_master_param = p.data.dtype in LOW_PRECISION_DTYPES
                if use_master_param:
                    if "fp32_param" not in param_state:
                        param_state["fp32_param"] = p.data.detach().float().clone()
                    elif param_state["fp32_param"].dtype != torch.float32:
                        param_state["fp32_param"] = param_state["fp32_param"].float()
                    param_data = param_state["fp32_param"]
                    grad_data = grad.float()
                else:
                    param_data = p.data
                    grad_data = grad

                if is_nora and grad.dim() >= 2:
                    if "momentum_buffer" not in param_state:
                        buf = torch.zeros_like(grad_data)
                    else:
                        buf = param_state["momentum_buffer"]
                        if use_master_param and buf.dtype != torch.float32:
                            buf = buf.float()

                    buf.lerp_(grad_data, 1 - beta)
                    m_t = grad_data.lerp(buf, momentum)

                    theta_hat = F.normalize(param_data, p=2, dim=-1, eps=eps)

                    dot_product = torch.sum(m_t * theta_hat, dim=-1, keepdim=True)
                    v = m_t - dot_product * theta_hat

                    v_hat = F.normalize(v, p=2, dim=-1, eps=eps)

                    scale = max(1, math.sqrt(grad_data.size(-2) / grad_data.size(-1)))
                    update_direction = v_hat * scale

                    if weight_decay > 0:
                        param_data.mul_(1 - lr * weight_decay)

                    param_data.add_(update_direction, alpha=-lr)

                    if use_master_param:
                        p.data.copy_(param_data.to(dtype=p.data.dtype))

                    param_state["momentum_buffer"] = buf

                else:
                    if "exp_avg" not in param_state:
                        param_state["exp_avg"] = torch.zeros_like(grad_data)
                        param_state["exp_avg_sq"] = torch.zeros_like(grad_data)
                        param_state["step"] = 0
                    elif use_master_param and param_state["exp_avg"].dtype != torch.float32:
                        param_state["exp_avg"] = param_state["exp_avg"].float()
                        param_state["exp_avg_sq"] = param_state["exp_avg_sq"].float()

                    exp_avg, exp_avg_sq = param_state["exp_avg"], param_state["exp_avg_sq"]
                    param_state["step"] += 1

                    exp_avg.mul_(betas[0]).add_(grad_data, alpha=1 - betas[0])
                    exp_avg_sq.mul_(betas[1]).addcmul_(grad_data, grad_data, value=1 - betas[1])

                    bias_correction1 = 1 - betas[0] ** param_state["step"]
                    bias_correction2 = 1 - betas[1] ** param_state["step"]
                    step_size = lr * math.sqrt(bias_correction2) / bias_correction1

                    denom = exp_avg_sq.sqrt().add_(eps)
                    adam_update = exp_avg / denom

                    if weight_decay > 0:
                        param_data.mul_(1 - step_size * weight_decay)

                    param_data.add_(adam_update, alpha=-step_size)

                    if use_master_param:
                        p.data.copy_(param_data.to(dtype=p.data.dtype))

        return loss


def get_nora_optimizer(
    model,
    lr_nora=0.005,
    lr_adam=0.001,
    weight_decay=0.1,
    momentum=0.95,
    beta=0.95,
):
    nora_params = []
    adam_params = []

    for name, param in model.named_parameters():
        if param.requires_grad:
            if param.ndim >= 2 and "embed" not in name and "lm_head" not in name:
                nora_params.append(param)
            else:
                adam_params.append(param)

    param_groups = [
        dict(
            params=nora_params,
            lr=lr_nora,
            lr_nora=lr_nora,
            lr_adam=lr_adam,
            weight_decay=weight_decay,
            momentum=momentum,
            beta=beta,
            is_nora=True,
        ),
        dict(
            params=adam_params,
            lr=lr_adam,
            lr_nora=lr_nora,
            lr_adam=lr_adam,
            weight_decay=weight_decay,
            momentum=momentum,
            beta=beta,
            is_nora=False,
        ),
    ]
    optimizer = Nora(param_groups)
    return optimizer
\end{lstlisting}

\end{document}